\documentclass{article}

\usepackage{microtype}
\usepackage{graphicx}
\usepackage{subfigure}
\usepackage{booktabs} % for professional tables

\usepackage[hyphens]{url}

\usepackage{hyperref}

\usepackage{xcolor}

\usepackage[accepted]{icmlArxiv}

\usepackage{layouts}

\usepackage{amsmath}
\usepackage{amssymb}

\usepackage{booktabs}

\usepackage{tikz}
\usetikzlibrary{decorations.pathreplacing, positioning, shapes, arrows.meta, math}

\usepackage{multirow}

\RequirePackage{rsfso}

\usepackage{nicefrac}

\usepackage{amsmath}

\usepackage{enumitem}

\usepackage{xcolor}

\usepackage{amsthm}
\newtheorem{theorem}{Theorem}

\usepackage{adjustbox}
\usepackage{rotating}

\newif\ifgenAppendices
\genAppendicestrue
\icmltitlerunning{ForecastNet: A Time-Variant Deep Feed-Forward Neural Network Architecture for Multi-Step-Ahead Time-Series Forecasting}
\begin{document}
    
    \twocolumn[
    \icmltitle{ForecastNet: A Time-Variant Deep Feed-Forward Neural Network Architecture for Multi-Step-Ahead Time-Series Forecasting}
    
    % It is OKAY to include author information, even for blind
    % submissions: the style file will automatically remove it for you
    % unless you've provided the [accepted] option to the icml2020
    % package.
    
    % List of affiliations: The first argument should be a (short)
    % identifier you will use later to specify author affiliations
    % Academic affiliations should list Department, University, City, Region, Country
    % Industry affiliations should list Company, City, Region, Country
    
    % You can specify symbols, otherwise they are numbered in order.
    % Ideally, you should not use this facility. Affiliations will be numbered
    % in order of appearance and this is the preferred way.
    \icmlsetsymbol{equal}{*}
    
    \begin{icmlauthorlist}
        \icmlauthor{Joel Janek Dabrowski}{addr1}
        \icmlauthor{YiFan Zhang}{addr2}
        \icmlauthor{Ashfaqur Rahman}{addr3}
    \end{icmlauthorlist}
    
    \icmlaffiliation{addr1}{Data61, CSIRO, Brisbane, QLD, 4067, Australia}
    \icmlaffiliation{addr2}{Agriculture \& Food, CSIRO, Brisbane, QLD, 4067, Australia}
    \icmlaffiliation{addr3}{Data61, CSIRO, Sandy Bay, TAS, 7005, Australia}
    
    \icmlcorrespondingauthor{Joel Janek Dabrowski}{joel.dabrowski@data61.csiro.au}
    
    % You may provide any keywords that you
    % find helpful for describing your paper; these are used to populate
    % the "keywords" metadata in the PDF but will not be shown in the document
    \icmlkeywords{deep learning, sequence models, seasonal, time series, forecasting, time invariance}
    
    \vskip 0.3in
    ]
    
    % this must go after the closing bracket ] following \twocolumn[ ...
    
    % This command actually creates the footnote in the first column
    % listing the affiliations and the copyright notice.
    % The command takes one argument, which is text to display at the start of the footnote.
    % The \icmlEqualContribution command is standard text for equal contribution.
    % Remove it (just {}) if you do not need this facility.
    
    %\printAffiliationsAndNotice{}  % leave blank if no need to mention equal contribution
%    \printAffiliationsAndNotice{\icmlEqualContribution} % otherwise use the standard text.
    \printAffiliationsAndNotice{} % otherwise use the standard text.
    
    \begin{abstract}
        Recurrent and convolutional neural networks are the most common architectures used for time series forecasting in deep learning literature. These networks use parameter sharing by repeating a set of fixed architectures with fixed parameters over time or space. The result is that the overall architecture is \textit{time-invariant} (shift-invariant in the spatial domain) or \textit{stationary}. We argue that time-invariance can reduce the capacity to perform multi-step-ahead forecasting, where modelling the dynamics at a \textit{range} of scales and resolutions is required. We propose ForecastNet which uses a deep feed-forward architecture to provide a \textit{time-variant} model. An additional novelty of ForecastNet is interleaved outputs, which we show assist in mitigating vanishing gradients. ForecastNet is demonstrated to outperform statistical and deep learning benchmark models on several datasets.
    \end{abstract}
    
%    \printinunitsof{in}\prntlen{\textwidth}
%    \printinunitsof{in}\prntlen{\columnwidth}
    %	\thefontsize{\scriptsize}
    
    %Joel Dabrowski
%ForecastNet
%Section: Introduction

\section{Introduction}

Multi-step-ahead forecasting involves the prediction of some variable several time-steps into the future, given past and present data. Over the set of time-steps, various time-series components such as complex trends, seasonality, and noise may be observed at a range of scales or resolutions. Increasing the number of steps ahead that are forecast increases the range of scales that are required to be modelled. An accurate forecasting method is required to model all these components over the complete range of scales.

There is significant interest in the recurrent neural network (RNN) and sequence-to-sequence models for forecasting \cite{kuznetsov2018foundations}. Several studies have shown success with variants of these models 
\cite{Zhu2017Deep, laptev2017time, wen2017multi, maddix2018deep, flunkert2017deepar, Xing2019Sentiment, hewamalage2019recurrent, Nguyen2019Forecasting, du2020multivariate}. 
The recurrence in the RNN produces a set of cells, each with fixed architecture, that are replicated over time. This replication results in a time-invariant model. Similarly, parameter sharing and convolution in the convolutional neural network (CNN) result in a shift-invariant model in the spatial domain; which is equivalent to time-invariance in the time domain. Our conjecture is that time-invariance can reduce the capacity for the model to learn various scales of dynamics across multiple steps in time; especially in multi-step-ahead forecasts.
%Our conjecture is that time-invariance can reduce the capacity for the model to learn the various scales of dynamics across multiple steps in   time (especially in multi-step-ahead forecasts) and that it contributes to the difficulty in learning long-term dependencies.

To address this, we propose ForecastNet, which is a deep feed-forward architecture with interleaved outputs. Between interleaved outputs are network structures that differ in architecture and parameters over time. The result is a model that varies across time. We demonstrate four variations of ForecastNet to highlight its flexibility. These four variations are compared with state-of-the-art benchmark models. We show that ForecastNet is accurate and robust in terms of performance variation.

The contributions of this study are: (1) ForecastNet: a model for multi-step-ahead forecasting; %a highly competitive alternative to the current state-of-the-art models; and combines key innovations from modern literature such as shortcut-connections, interleaved outputs, and mixture density networks, in a novel way; 
(2) we address the time-invariance problem which (to our knowledge) has not been considered in deep learning time-series forecasting literature before; and (3) provide a comparison of several state-of-the-art models on seasonal time-series datasets.

    %Joel Dabrowski
%Tapped Deep Network
%Section: Model

\section{Motivations and Related Work}

\subsection{Recurrent Neural Networks}

The RNN comprises a set of cell structures with parameters that are replicated over time. These replications (cells and parameters) remain constant over time. The replication has its benefits (such as parameter sharing) but it can reduce capacity to model the complex dependencies over time. For example, a model can adapt to long-term dynamics more easily if its parameters are able to change over time.

Another challenge with RNNs is that learning long sequences can be difficult due to complex dependencies over time and vanishing gradients \citep{Chang2017Dilated}. Vanishing gradients in the RNN have been addressed in the LSTM \citep{hochreiter1997long, gers1999learning} and the Gated Recurrent Unit (GRU) \citep{cho2014learning, cho2014properties}, by introducing gate-like structures. These LSTM and GRU cell structures are however recurrent and are constant over time.

ForecastNet comprises a set of parameters and an architecture that changes over the sequence of inputs and forecast outputs. The result is that ForecastNet is \textit{not} a time-invariant model. Furthermore, ForecastNet mitigates vanishing gradient problems by using shortcut-connections and by interleaving outputs between hidden cells. The shortcut-connection approach was introduced in ResNet \citep{He2015Deep} and Highway Network \citep{Srivastava2015Training}. It has been shown to be highly effective in addressing vanishing and shattered gradient problems \citep{balduzzi2017shattered}. Additionally, shortcut-connections allow for a much deeper network \citep{He2015Deep}.

\subsection{Sequence Models}

State-of-the-art deep sequence models include multiple RNNs linked in various configurations \citep{Zhang2019SSIM}. A prominent configuration is the sequence-to-sequence (encoder-decoder) model \citep{Sutskever2014Sequence}. The sequence-to-sequence model sequentially links two RNNs (an encoder and a decoder) through a fixed size vector, such as the last encoder cell state. This can be limiting as it forms a potential bottleneck between the encoder and decoder. Furthermore, earlier inputs have to pass through several layers to reach the decoder. 

The attention model \citep{bahdanau2014neural} addresses the sequence-to-sequence model problem by adding a network structure from all the encoder cells to each decoder cell. This structure, called an \textit{attention mechanism}, ascribes relevance to the particular encoder cell. The attention mechanism is not necessarily time-invariant, so it may help reduce the time-invariance properties of the overall model. However, the time series dynamics are primarily modelled with the RNN encoder and decoder, which are time invariant.

Unlike the sequence-to-sequence and attention models, ForecastNet does not have a separate encoder and decoder. Challenges with linking these entities are thus removed.

\subsection{Convolutional Neural Networks}

The convolutional neural network (CNN) has shown promising results in modelling sequential data \citep{binkowski2018autoregressive, Mehrkanoon2019Deep, sen2019think}. The CNN uses convolution and parameter sharing to achieve shift-invariance (or translation-invariance) \cite{lecun1995convolutional}. This is a key feature in image processing, however, in time-series applications, it translates to time-invariance.

WaveNet \cite{oord2016wavenet} is a seminal CNN model which uses multiple layers of dilated causal convolutions for raw audio synthesis. This model has been generalised for broader sequence modelling problems and this generalisation is referred to as the Temporal Convolutional Network (TCN) \cite{bai2018empirical}. We select this model as a benchmark for comparing ForecastNet. Furthermore, we demonstrate how ForecastNet is able to accommodate convolutional layers in hidden cells.

\subsection{The Transformer and Self-Attention}

A model that has successfully departed from the RNN and CNN architectures is the transformer model \cite{vaswani2017attention}. This model comprises a sequence of encoders and decoders. The encoders include a self-attention mechanism and a position-wise feed-forward neural network. The decoder has the same architecture as the encoder, but with an additional attention mechanism over the encoding. 

Though the transformer has been highly successful in natural language processing, it has limitations in time series analysis. The first limitation is that it does not assume any sequential structure of the inputs \cite{vaswani2017attention}. Positional encoding in the form of a sinusoid is injected into the inputs to provide information on the sequence order. Temporal structure is key in time series modelling and is what time-series models are usually designed to model. \citet{li2019enhancing} propose including causal convolution to model local context. Convolutions are however time-invariant.

A second limitation of the transformer is that (to promote parallelisation) a large portion of the processing operates over the dimension of the input embedding rather than over time. For example, the multiple attention heads and the position-wise feed-forward neural networks operate over the embedding dimension. The transformer is thus \textit{not} designed to operate on low dimensional time series signals such as univariate time-series as considered in this study. \citet{wu2020deep} use a feed-forward neural network to increase the dimensionality of the input space. We however did not find such an approach effective on the datasets used in this study.

ForecastNet addresses both transformer model limitations. It is specifically designed to model the temporal structure of the inputs and it is not limited to multivariate time-series.

\subsection{Uncertainty Representation}

Realistically, forecasting methods should provide some form of uncertainty or confidence interval around estimates \cite{Makridakis2018Statistical}. Neural networks do not naturally provide this capability. Several approaches to incorporate uncertainty have however been proposed. These include: empirical approaches (such as bootstrapped residuals \cite{hyndman2018forecasting}, or Monte Carlo dropout \cite{Zhu2017Deep}; ensembles (such as \cite{lakshminarayanan2017Simple}); variational inference-based models (such as \cite{bayer2014learning, krishnan2015deep, chung2015recurrent, fraccaro2016sequential, doerr2018probabilistic, rangapuram2018deep}); or predictive distributions (such as \cite{flunkert2017deepar}). The predictive distribution approach (based on mixture density networks \citep{bishop1994mixture, graves2013generating}) is effective as the distribution parameters are learned directly through gradient descent. ForecastNet thus adopts this mixture density approach.

    %Joel Dabrowski
%ForecastNet
%Section: ForecastNet

\section{ForecastNet Architecture}

As illustrated in \figurename{~\ref{fig:forecastNet}}, ForecastNet is a \textit{feed-forward} neural network comprising a set of $n_I$ inputs, a set of $n_O$ outputs, and a set of sequentially connected hidden \textit{cells} (a term borrowed from RNN literature). A detailed diagram of a \textit{simple} form of ForecastNet is presented in \figurename{~\ref{fig:forecastNetBasic}}. %Note that ForecastNet is \textit{not} a RNN by definition because it does \textit{not} have any recurrent connections.
%
%\begin{figure*}_t{}
%	\centering
%	\begin{subfigure}_t{\columnwidth}
%		\centering
%		\input{figures/fig_forecastNet2}
%		\caption{General ForecastNet structure to provide a forecast $\hat{x}_{t+1:t+4}$ given $\mathbf{x}=x_{t-n_I+1:t}$ as inputs (circles). Hidden cells (squares) comprise some form of feed forward neural network structure. Outputs (diamonds) are probabilistic. Each hidden cell and output is illustrated with a different shade to indicate heterogeneity over the sequence.} 
%		\label{fig:forecastNet}
%	\end{subfigure} \hfill
%	\centering
%	\begin{subfigure}_t{\columnwidth}
%		\centering
%		\input{figures/fig_forecastNetBasic}
%		\caption{An example of a simple form of ForecastNet with a single densely connected hidden layer in the hidden cell (more complex structures can be chosen for the hidden cell). Outputs take the form of a normal or Gaussian distribution. Note that though the architecture at each sequence step is identical, the weights will differ (i.e. they are \textit{not} recurrent).}
%		\label{fig:forecastNetBasic}
%	\end{subfigure}
%	\caption{ForecastNet architecture.}
%	\label{fig:forecastNetFig}
%\end{figure*}
%
%
\begin{figure}[t]{}
    \centering
    %Joel Dabrowski
%Tikz Figure
%Name: fig_forecastNetGeneral
%ForecastNet general structure

%\def\horizontalsep{1.7cm}
\def\horizontalsep{1.1cm}
\def\verticalsep{0.28cm}

\begin{scriptsize}
\begin{tikzpicture}[->, >={Latex[length=1.5mm, width=1mm, black!40]}, draw=black!40]
%\tikzstyle{every pin edge}=[<-,thick]
%\tikzstyle{output}=[circle, fill=black!20, draw=black!60, minimum size=0.3cm,inner sep=0pt]
\tikzstyle{output}=[diamond, fill=black!35, draw=black!60, minimum size=0.4cm,inner sep=0pt]
%\tikzstyle{input}=[circle, fill=black!40, draw=black!60, minimum size=0.3cm,inner sep=0pt]
\tikzstyle{input}=[circle, fill=black!20, draw=black!60, minimum size=0.3cm,inner sep=0pt]
\tikzstyle{neuron}=[fill=black!55, minimum size=9pt,inner sep=0pt]
\pgfsetshortenstart{0.7pt}
\pgfsetshortenend{0.7pt}

%Inputs
\node[input, pin={[pin edge={<-}]above:$\mathbf{x}$}] (X1) at (0*\horizontalsep, 0.5*\verticalsep) {};
\node[input, pin={[pin edge={<-}]above:$\mathbf{x}$}] (X2) at (1*\horizontalsep, 0.5*\verticalsep) {};
\node[input, pin={[pin edge={<-}]above:$\mathbf{x}$}] (X3) at (2*\horizontalsep, 0.5*\verticalsep) {};
\node[input, pin={[pin edge={<-}]above:$\mathbf{x}$}] (X4) at (3*\horizontalsep, 0.5*\verticalsep) {};

%Hidden layer 1
\node[neuron,fill=black!70] (H1) at (0*\horizontalsep, -2*\verticalsep) {};
%Hidden layer 2
\node[neuron,fill=black!60] (H2) at (1*\horizontalsep, -2*\verticalsep) {};
%Hidden layer 3
\node[neuron,fill=black!50] (H3) at (2*\horizontalsep, -2*\verticalsep) {};
%Hidden layer 4
\node[neuron,fill=black!40] (H4) at (3*\horizontalsep, -2*\verticalsep) {};

%Outputs
\node[output,fill=black!70, pin={[pin edge={->}]below:$\hat{x}_{t+1}$}] (Y1) at (0.5*\horizontalsep, -4*\verticalsep) {};
\node[output,fill=black!60, pin={[pin edge={->}]below:$\hat{x}_{t+2}$}] (Y2) at (1.5*\horizontalsep, -4*\verticalsep) {};
\node[output,fill=black!50, pin={[pin edge={->}]below:$\hat{x}_{t+3}$}] (Y3) at (2.5*\horizontalsep, -4*\verticalsep) {};
\node[output,fill=black!40, pin={[pin edge={->}]below:$\hat{x}_{t+4}$}] (Y4) at (3.5*\horizontalsep, -4*\verticalsep) {};

%Links from input to hidden
\foreach \source in {1,...,4}
%\foreach \dest in {1,...,4}
\path (X\source) edge (H\source.north);

%Links from hidden to hidden layers
\path (H1) edge (H2);
\path (H2) edge (H3);
\path (H3) edge (H4);

%Links from hidden to output layers
\path (H1.south east) edge (Y1);
\path (H2.south east) edge (Y2);
\path (H3.south east) edge (Y3);
\path (H4.south east) edge (Y4);

\path (Y1) edge (H2.south west);
\path (Y2) edge (H3.south west);
\path (Y3) edge (H4.south west);

% Inputs, hidden cells and outputs labels
\draw [-, decorate,decoration={brace, amplitude=5pt,raise=0pt},black] 
(4*\horizontalsep,3.5*\verticalsep) -- (4*\horizontalsep,0.0*\verticalsep) 
node [black, midway, xshift=8pt, right, label={[rotate=90]center:Inputs}] {};

\draw [-, decorate,decoration={brace, amplitude=5pt,raise=0pt},black] 
(4*\horizontalsep,-0.5*\verticalsep) -- (4*\horizontalsep,-3.0*\verticalsep) 
node [black,midway, xshift=12pt, right, label={[rotate=90,align=center]center:{Hidden \\ cells}}] {};

\draw [-, decorate,decoration={brace, amplitude=5pt,raise=0pt},black] 
(4*\horizontalsep,-3.5*\verticalsep) -- (4*\horizontalsep,-7.4*\verticalsep) 
node [black,midway,xshift=8pt, right, label={[rotate=90]center:Outputs}] {};

\end{tikzpicture}
\end{scriptsize}
    \caption{General ForecastNet structure to provide a forecast $\hat{x}_{t+1:t+4}$ given $\mathbf{x}=x_{t-n_I+1:t}$ as inputs (circles). A hidden cell (squares) comprises some form of feed forward neural network structure. Each hidden cell and output is illustrated with a different shade to indicate heterogeneity over the sequence.} 
    \label{fig:forecastNet}
\end{figure}
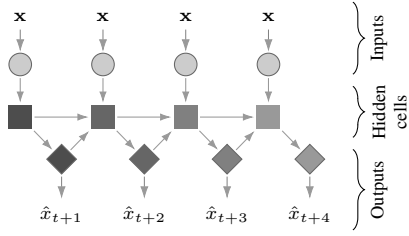 
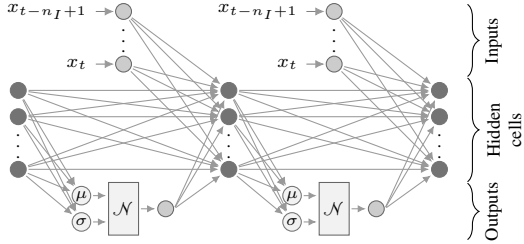
\begin{figure}[t]{}
    \centering
    %Joel Dabrowski
%Tikz Figure
%Name: fig_weightsIllustration
%ForecastNet showing the weight matrices with their corresponding links

%Joel Dabrowski
%Tikz Figure
%Name: fig_weightsIllustration
%ForecastNet showing the weight matrices with their corresponding links

% Small figure without weight labels

\def\horisep{0.7cm}
\def\vertsep{0.35cm}

\begin{scriptsize}
	\begin{tikzpicture}[->, >={Latex[length=1.1mm, width=1mm, black!40]}, draw=black!40, node distance=\horisep]
	\tikzstyle{neuron}=[circle,fill=black!55,draw=black!55,minimum size=6pt,inner sep=0pt]
	\tikzstyle{param}=[circle,fill=black!5,draw=black!55,minimum size=7pt,inner sep=0pt]
	\tikzstyle{input}=[neuron, fill=black!20, draw=black!60,];
	\tikzstyle{output}=[neuron, fill=black!20, draw=black!60];
	\tikzstyle{likelihood}=[rectangle, fill=black!5, draw=black!60, minimum width=9pt, minimum height=2*\vertsep, inner sep=2pt];
	\tikzstyle{dots}=[minimum size=9pt,inner sep=0pt]
	\tikzstyle{vdots}=[minimum size=9pt,inner sep=0pt, label={[yshift=-0.45cm]$\vdots$}]
	\tikzstyle{annot} = [text width=4em, text centered]
	\pgfsetshortenstart{0.7pt}
	\pgfsetshortenend{0.7pt}

	% Layer 1: Hidden layer 1
	\node[neuron] (L1-0) at (0,-0*\vertsep) {};
	\node[neuron] (L1-1) at (0,-1*\vertsep) {};
	\node[vdots]  (L1-2) at (0,-2*\vertsep) {};
	\node[neuron] (L1-3) at (0,-3*\vertsep) {};
	
	% Layer 2: Input layer
	\node[input, pin={[pin distance=8pt, pin edge={<-}]left:$x_{t-n_I+1}$}] (L2-0) at (2*\horisep, 3*\vertsep) {};
	\node[vdots] (L2-1) at (2*\horisep,2*\vertsep) {};
	\node[input, pin={[pin distance=8pt, pin edge={<-}]left:$x_{t}$}] (L2-2) at (2*\horisep, 1*\vertsep) {};
%	\node[annot] (label) at (1.6*\horisep, 2*\vertsep) {$\mathbf{x}$};
	
	% Layer 2: Likelihood and output
	\node[param] (L2-3) at (1.2*\horisep, -4*\vertsep) {{\tiny$\mu$}};
	\node[param] (L2-4) at (1.2*\horisep, -5*\vertsep) {{\tiny$\sigma$}};
	\node[likelihood] (L2-5) at (2*\horisep, -4.5*\vertsep) {{\scriptsize$\mathcal{N}$}};
	\node[output] (L2-6) at (2.8*\horisep, -4.5*\vertsep) {};
	%	\node[annot] (label) at (1.3*\horisep, -6*\vertsep) {$A^{l-1}$};
	
	% Links layer 1 to layer 2
	\path (L1-0) edge (L2-3);
	\path (L1-1) edge (L2-3);
	\path (L1-3) edge (L2-3);
	
	\path (L1-0) edge (L2-4);
	\path (L1-1) edge (L2-4);
	\path (L1-3) edge (L2-4);
	
	% Links from likelihood 
	\draw[->] (L2-3.east) -- (L2-3.east -| L2-5.west);
	\draw[->] (L2-4.east) -- (L2-4.east -| L2-5.west);
	\path (L2-5) edge (L2-6);

	% Layer 1: Hidden layer 3
	\node[neuron] (L3-0) at (4*\horisep,-0*\vertsep) {};
	\node[neuron] (L3-1) at (4*\horisep,-1*\vertsep) {};
	\node[vdots] (L3-2) at (4*\horisep,-2*\vertsep) {};
	\node[neuron] (L3-3) at (4*\horisep,-3*\vertsep) {};
	
	% Layer 4: Input layer
	\node[input, pin={[pin distance=8pt, pin edge={<-}]left:$x_{t-n_I+1}$}] (L4-0) at (6*\horisep, 3*\vertsep) {};
	\node[vdots]  (L4-1) at (6*\horisep,2*\vertsep) {};
	\node[input, pin={[pin distance=8pt, pin edge={<-}]left:$x_{t}$}] (L4-2) at (6*\horisep, 1*\vertsep) {};
%	\node[annot] (label) at (5.6*\horisep, 2*\vertsep) {$\mathbf{x}$};
	
	% Layer 4: Tapped output layer
	\node[param] (L4-3) at (5.2*\horisep, -4*\vertsep) {{\tiny$\mu$}};
	\node[param] (L4-4) at (5.2*\horisep, -5*\vertsep) {{\tiny$\sigma$}};
	\node[likelihood] (L4-5) at (6*\horisep, -4.5*\vertsep) {{\scriptsize$\mathcal{N}$}};
	\node[output] (L4-6) at (6.8*\horisep, -4.5*\vertsep) {};
	%	\node[annot] (label) at (3.3*\horisep, -6*\vertsep) {$A^{l+1}$};
	
	% Links layer 1 to layer 3
	\path (L1-0) edge (L3-0);
	\path (L1-1) edge (L3-0);
	\path (L1-3) edge (L3-0);
	\path (L1-0) edge (L3-1);
	\path (L1-1) edge (L3-1);
	\path (L1-3) edge (L3-1);
	\path (L1-0) edge (L3-3);
	\path (L1-1) edge (L3-3);
	\path (L1-3) edge (L3-3);
	
	% Links layer 2 to layer 3
	\path (L2-0) edge (L3-0);
	\path (L2-0) edge (L3-1);
	\path (L2-0) edge (L3-3);
	
	\path (L2-2) edge (L3-0);
	\path (L2-2) edge (L3-1);
	\path (L2-2) edge (L3-3);
	
	\path (L2-6.east) edge (L3-0);
	\path (L2-6.east) edge (L3-1);
	\path (L2-6.east) edge (L3-3);

	% Links layer 3 to layer 4
	\path (L3-0) edge (L4-3);
	\path (L3-1) edge (L4-3);
	\path (L3-3) edge (L4-3);
	
	\path (L3-0) edge (L4-4);
	\path (L3-1) edge (L4-4);
	\path (L3-3) edge (L4-4);
	
	% Links from likelihood 
	\draw[->] (L4-3.east) -- (L4-3.east -| L4-5.west);
	\draw[->] (L4-4.east) -- (L4-4.east -| L4-5.west);
	\path (L4-5) edge (L4-6);

	% Layer 5: Hidden layer 5
	\node[neuron] (L5-0) at (8*\horisep,-0*\vertsep) {};
	\node[neuron] (L5-1) at (8*\horisep,-1*\vertsep) {};
	\node[vdots] (L5-2) at (8*\horisep,-2*\vertsep) {};
	\node[neuron] (L5-3) at (8*\horisep,-3*\vertsep) {};
	
	% Links layer 3 to layer 5
	\path (L3-0) edge (L5-0);
	\path (L3-1) edge (L5-0);
	\path (L3-3) edge (L5-0);
	\path (L3-0) edge (L5-1);
	\path (L3-1) edge (L5-1);
	\path (L3-3) edge (L5-1);
	\path (L3-0) edge (L5-3);
	\path (L3-1) edge (L5-3);
	\path (L3-3) edge (L5-3);
	
	% Links layer 4 to layer 5
	\path (L4-0) edge (L5-0);
	\path (L4-0) edge (L5-1);
	\path (L4-0) edge (L5-3);
	
	\path (L4-2) edge (L5-0);
	\path (L4-2) edge (L5-1);
	\path (L4-2) edge (L5-3);
	
	\path (L4-6.east) edge (L5-0);
	\path (L4-6.east) edge (L5-1);
	\path (L4-6.east) edge (L5-3);
	
	\definecolor{labelColor}{rgb}{0.0,0.0,0.0} %black

	%	% Weight labels
	%	\draw[labelColor,thick,dashed] (1.6*\horisep,-1.5*\vertsep) ellipse(0.1 cm and 3.5*\vertsep);
	%	\node[labelColor]() at (1.6*\horisep+0.18cm, 2.6*\vertsep) {$W^{[l]}$};
	%	
	%	\draw[labelColor,thick,dashed] (3.6*\horisep,-1.5*\vertsep) ellipse(0.1 cm and 3.5*\vertsep);
	%	\node[labelColor]() at (3.6*\horisep+0.34cm, 2.6*\vertsep) {$W^{[l+2]}$};
	%	
	%	\draw[rotate=-35,labelColor,thick,dashed] (2.42*\horisep, 0.9*\horisep) ellipse(0.07 cm and 0.35 cm);
	%	\node[labelColor, align=center]() at (2.3*\horisep-0.2cm, -5.5*\vertsep) {$W_{\mu}^{[l+1]}$, \\ $W_{\sigma}^{[l+1]}$};

	% Inputs, hidden cells and outputs labels
	\draw [-, decorate,decoration={brace, amplitude=5pt,raise=0pt},black] 
	(8.5*\horisep,3.3*\vertsep) -- (8.5*\horisep,0.5*\vertsep) 
	node [black, midway, xshift=8pt, right, label={[rotate=90]center:Inputs}] {};
	
	\draw [-, decorate,decoration={brace, amplitude=5pt,raise=0pt},black] 
	(8.5*\horisep,0.5*\vertsep) -- (8.5*\horisep,-3.5*\vertsep) 
	node [black,midway, xshift=12pt, right, label={[rotate=90,align=center]center:{Hidden \\ cells}}] {};
	
	\draw [-, decorate,decoration={brace, amplitude=5pt,raise=0pt},black] 
	(8.5*\horisep,-3.5*\vertsep) -- (8.5*\horisep,-5.7*\vertsep) 
	node [black,midway,xshift=8pt, right, label={[rotate=90]center:Outputs}] {};

%	% Layers labels
%	\draw [-, decorate,decoration={brace, amplitude=5pt,raise=0pt},black] 
%	(0.5*\horisep,-6*\vertsep) -- (-0.5*\horisep,-6*\vertsep) 
%	node [black,midway,below, yshift=-3pt] {$l-2$};
%	
%	\draw [-, decorate,decoration={brace, amplitude=5pt,raise=0pt},black] 
%	(3.4*\horisep,-6*\vertsep) -- (0.6*\horisep,-6*\vertsep) 
%	node [black,midway,below, yshift=-3pt] {$l-1$};
%	
%	\draw [-, decorate,decoration={brace, amplitude=5pt,raise=0pt},black] 
%	(4.5*\horisep,-6*\vertsep) -- (3.5*\horisep,-6*\vertsep) 
%	node [black,midway,below, yshift=-3pt] {$l$};
%	
%	\draw [-, decorate,decoration={brace, amplitude=5pt,raise=0pt},black] 
%	(7.4*\horisep,-6*\vertsep) -- (4.6*\horisep,-6*\vertsep) 
%	node [black,midway,below, yshift=-3pt] {$l+1$};
%	
%	\draw [-, decorate,decoration={brace, amplitude=5pt,raise=0pt},black] 
%	(8.5*\horisep,-6*\vertsep) -- (7.5*\horisep,-6*\vertsep) 
%	node [black,midway,below, yshift=-3pt] {$l+2$};
	\end{tikzpicture}
\end{scriptsize}
    \caption{An example of a simple form of ForecastNet with a single densely connected hidden layer in the hidden cell (more complex structures can be chosen for the hidden cell). Outputs take the form of a normal or Gaussian distribution. Note that though the architecture at each sequence step is identical, the weights will differ (i.e. they are \textit{not} recurrent).}
    \label{fig:forecastNetBasic}
\end{figure}

\subsection{Inputs}

ForecastNet's inputs $\mathbf{x}=x_{t-n_I+1:t}$ are a set of lagged values of the dependent variable. The dependent variable can be univariate or multivariate. The set of inputs are presented to every hidden cell in the network as illustrated in \figurename{~\ref{fig:forecastNet}} and \figurename{~\ref{fig:forecastNetBasic}}.

\subsection{Hidden Cells}

A hidden cell represents some form of feed forward neural network such as a multi-layered perceptron (MLP), a CNN, or self-attention. Each hidden cell can be heterogeneous in terms of architecture. As a feed-forward network, even if the architecture of each hidden cell is identical (as used in this study), each cell is provided with its own unique set of parameters. This is in contrast to RNNs where the RNN cell architecture and parameters are duplicated at each sequence step. ForecastNet is thus sequential, but it is \textit{not} recurrent.

The hidden cells are intended to model the time-series dynamics. Links between hidden cells model local dynamics and cells together model longer-term dynamics. %\textcolor{red}{The receptive field is very wide}.

%As illustrated in \figurename{~\ref{fig:forecastNet}}, a hidden cell's input comprises the output of the previous hidden cell, the output of the previous output layer, and the inputs $\mathbf{x}$. The output of the hidden cell is connected to the following hidden cell as well as an output layer.

\subsection{Outputs}

Each output in ForecastNet provides a forecast one-step into the future. The deeper the network, the more outputs there are. ForecastNet thus \textit{naturally scales in complexity with increased forecast reach}. 

Using the idea of mixture density networks \citep{flunkert2017deepar, bishop1994mixture}, each output models a probability distribution. In this study the normal distribution is used as illustrated in \figurename{~\ref{fig:forecastNetBasic}}. The mean and standard deviation of the normal distribution at output layer $l$ are given by
\begin{align}
&\mu^{[l]} = W_{\mu}^{[l]T} \mathbf{a}^{[l-1]} + \mathbf{b}^{[l]}_{\mu} \\
&\sigma^{[l]} = \log( 1 + \exp(W_{\sigma}^{[l]T} \mathbf{a}^{[l-1]} + \mathbf{b}^{[l]}_{\sigma})),
\end{align}
where $\mathbf{a}^{[l-1]}$ are the outputs of the previous hidden cell, $W_{\mu}^{[l]T}$ and $\mathbf{b}^{[l]}_{\mu}$ are the weights and biases of the mean's layer, and $W_{\sigma}^{[l]T}$ and $\mathbf{b}^{[l]}_{\sigma}$ are the weights and biases of the standard deviation's layer. 

The forecast associated with layer $l$ is produced by sampling from $\mathcal{N}(\mu^{[l]},\sigma^{[l]})$. During forecasting, the sampled forecast is fed to the next layer. During training, the forecast is fully observable through the training data. The network is trained with gradient descent to optimise the normal distribution log-likelihood function \citep{flunkert2017deepar}.

The mixture density output can be replaced with a linear output and the squared error loss function. No uncertainty will be available in this form, however the model be required to optimise over two parameters. This form is demonstrated in this study as one of the variations of ForecastNet.

\section{ForecastNet Properties}
%\section{Interleaved Outputs, Convergence, and Vanishing Gradients}
%\section{Effects of Interleaved Outputs}%{Interleaved Outputs and Vanishing Gradient}

\subsection{Time-Variance}

A time-invariant system is defined as \textit{a system for which a time shift of the input sequence causes a corresponding shift in the output sequence} \cite{oppenheim2009discrete}. 

\begin{theorem}
    ForecastNet with inputs $\mathbf{x}_t$, hidden states $\mathbf{h}_t$, and outputs $\mathbf{y}_t$ given by
    \begin{align}
    \label{eq:fn1}
    & \mathbf{y}_t = f_t(\mathbf{h}_t) \\ 
    \label{eq:fn2}
    & \mathbf{h}_t = g_t(\mathbf{x}_t, \mathbf{h}_{t-1}, \mathbf{y}_{t-1})
    \end{align}
    is not time invariant.
\end{theorem}
\begin{proof}
    Given two inputs, $\mathbf{x}_t$ and $\mathbf{x'}_t$, two outputs $\mathbf{y}_t$ and $\mathbf{y'}_t$, and two hidden states $\mathbf{h}_t$ and $\mathbf{h'}_t$. Let $\mathbf{x'}_t$ be $\mathbf{x}_t$, shifted by $t_0$ such that $\mathbf{x'}_t = \mathbf{x}_{t-t_0}$. Similarly, let $\mathbf{h'}_t = \mathbf{h}_{t-t_0}$. Time-invariance requires that $\mathbf{y}_{t-t_0} = \mathbf{y'}_t$. From (\ref{eq:fn1}) and (\ref{eq:fn2}), $\mathbf{y}_{t-t_0}$ is given by
    \begin{align*}
    \mathbf{y}_{t-t_0} = 
    f_{t-t_0}(g_{t-t_0}(\mathbf{x}_{t-t_0}, \mathbf{h} _{t-t_0-1}, \mathbf{y}_{t-t_0-1}))
    \end{align*}
    and $\mathbf{y'}_t$ is given by
    \begin{align*}
    \mathbf{y'}_t &= f_t(g_t(\mathbf{x'}_t, \mathbf{h'}_{t-1}, \mathbf{y'}_{t-1})) \\
    & = f_t(g_t(\mathbf{x}_{t-t_0}, \mathbf{h}_{t-t_0-1}, \mathbf{y'}_{t-1})).
    \end{align*}
    Thus, $\mathbf{y}_{t-t_0} \neq \mathbf{y'}_t$ and ForecastNet is not time-invariant.
\end{proof}
ForecastNet is not time-invariant as its parameters (and optionally architecture) vary in time (over the sequence of inputs and outputs). This is compared with the RNN that has fixed parameters which are reused each time step, resulting in a time-invariant model.

\begin{theorem}
    A RNN with inputs $\mathbf{x}_t$, hidden states $\mathbf{h}_t$, and outputs $\mathbf{y}_t$ given by
    \begin{align}
    \label{eq:rnn1}
    & \mathbf{y}_t = f(\mathbf{h}_t) \\ 
    \label{eq:rnn2}
    & \mathbf{h}_t = g(\mathbf{x}_t, \mathbf{h}_{t-1})
    \end{align}
    is time-invariant when $\mathbf{h}_t$ is initialised to some initial value $\mathbf{h}_0$ immediately before the first input is received.
\end{theorem}
\begin{proof}
    Given two inputs, $\mathbf{x}_t$ and $\mathbf{x'}_t$, two outputs $\mathbf{y}_t$ and $\mathbf{y'}_t$, and two hidden states $\mathbf{h}_t$ and $\mathbf{h'}_t$. Let $\mathbf{x'}_t$ be $\mathbf{x}_t$, shifted by $t_0$ such that $\mathbf{x'}_t = \mathbf{x}_{t-t_0}$. Similarly, let $\mathbf{h'}_t = \mathbf{h}_{t-t_0}$. Time-invariance requires that $\mathbf{y}_{t-t_0} = \mathbf{y'}_t$. From (\ref{eq:rnn1}) and (\ref{eq:rnn2}), $\mathbf{y}_{t-t_0}$ is given by
    \begin{align*}
    \mathbf{y}_{t-t_0} = f(g(\mathbf{x}_{t-t_0}, \mathbf{h}_{t-t_0-1}))
    \end{align*}
    and $\mathbf{y'}_t$ is given by
    \begin{align*}
    \mathbf{y'}_t &= f(g(\mathbf{x'}_t, \mathbf{h'}_{t-1})) \\
    & = f(g(\mathbf{x}_{t-t_0}, \mathbf{h}_{t-t_0-1})).
    \end{align*}
    Thus $\mathbf{y}_{t-t_0} = \mathbf{y'}_t$ and the RNN is time-invariant.
\end{proof}

Note that a requirement for time-invariance in the RNN is that $\mathbf{h}_t$ is initialised to $\mathbf{h}_0$ before the first input is received. If $\mathbf{h}_t$ were initialised several time-steps before the first input was received, a set of zeros (padding) would have to be provided as inputs until the first $\mathbf{x}_t$ were received. This could result in $\mathbf{h'}_t \neq \mathbf{h}_{t-t_0}$.

%%
% Elman Network
%\textbf{Proof:} Consider the Elman RNN given by:
%\begin{align*}
%& \mathbf{h}_t = g \left( W \mathbf{h}_{t-1} + U x_t + b_h \right) \\
%& \mathbf{y}_t = g \left( V \mathbf{h}_t + b_y \right)
%\end{align*}
%%
%Such that
%%
%\begin{align*}
%& \mathbf{y}_t = g \left( V g \left( W \mathbf{h}_{t-1} + U x_t + b_h \right) + b_y \right)
%\end{align*}
%%
%Given two inputs, $\mathbf{x}_t$ and $\mathbf{x'}_t$, two outputs $\mathbf{y}_t$ and $\mathbf{y}_t$, and two corresponding hidden states $\mathbf{h}_t$, $\mathbf{h'}_t$ such that $\mathbf{x'}_t = \mathbf{x}_{t-t_0}$ and $\mathbf{h'}_t = \mathbf{h}_{t-t_0}$. For time-invariance $\mathbf{y}_{t-t_0} = \mathbf{y'}_t$.
%%
%\begin{align*}
%& \mathbf{y}_{t-t_0} = g \left( V g \left( W \mathbf{h}_{t-t_0-1} + U \mathbf{x}_{t-t_0} + b_h \right) + b_y \right)
%\end{align*}
%%
%Let $\mathbf{x'}_t = \mathbf{x}_{t-t_0}$ and $\mathbf{h'}_t = \mathbf{h}_{t-t_0}$ such that
%\begin{align*}
%& \mathbf{y'}_t = g \left( V g \left( W \mathbf{h'}_{t-1} + U \mathbf{x'}_t + b_h \right) + b_y \right) \\
%& \mathbf{y'}_t = g \left( V g \left( W \mathbf{h}_{t-t_0-1} + U \mathbf{x'}_{t-t_0} + b_h \right) + b_y \right)
%\end{align*}
%Thus $\mathbf{y}_{t-t_0} = \mathbf{y'}_t$ which indicates time invariance.

\subsection{Interleaved Outputs}
\label{sec:effectsOfInterleavedOutputs}
The vanishing/exploding gradient problem has been referred to as ``deep learning's fundamental problem'' \citep{Schmidhuber2015Deep}. The problem stems from the repeated application of the chain rule of calculus in computing the gradient \citep{Caterini2018Deep}. The chain rule produces a chain of \textit{factors}. Especially long chains associated with deep networks can either vanish to zero or diverge (explode). By \textit{interleaving outputs} between hidden layers in ForecastNet, the chain is broken into a \textit{sum of terms}. This sum of terms is more stable than a product of factors. The intuition is that interleaved outputs provide localised information to inner hidden layers of the network during training. This decreases the effective depth of the network.
    
To show this, consider an $L$-layered ForecastNet with a single hidden neuron and a single output neuron. A summation results when the output of hidden layer is split between the next hidden layer and the next output. The repeated application of the chain rule results in the following expression
\begin{align}
\label{eq:chainRule2}
\frac{\partial \mathcal{L}}{\partial W^{[l]}} 
&= 	
\sum_{k=0}^{\frac{L-1-l}{2}}
\frac{\partial \mathcal{L}}{\partial \mathbf{z}^{[l+2k+1]}} 
\frac{\partial \mathbf{z}^{[l+2k+1]}}{\partial \mathbf{a}^{[l+2k]}}
\Psi_k
\frac{\partial \mathbf{a}^{[l]}}{\partial W^{[l]}}
\end{align}
where
\begin{align}
%~~ \text{where} ~~
\Psi_k = 
\begin{cases}
1 & k=0 \\
\displaystyle \prod_{j=1}^{k} \frac{\partial \mathbf{z}^{[l+2j]}}{\partial \mathbf{a}^{[l+2(j-1)]}}  & k>0
\end{cases}
\end{align}
\begin{proof}
    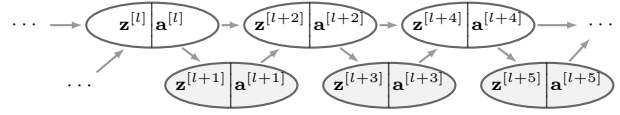
\begin{figure}[t]
        \centering
        %Joel Dabrowski
%Tikz Figure
%Name: fig_chainRuleDerivation
%ForecastNet figure for the derivation of the chain rule of calculus

\def\horizontalsep{1.05cm}
\def\verticalsep{0.8cm}

\begin{scriptsize}
\begin{tikzpicture}[->,>={Latex[length=1.5mm, width=1mm, black!40]},draw=black!40, thick]
	\tikzstyle{every pin edge}=[<-]
	\tikzstyle{neuron}=[ellipse, fill=black!0, draw=black!60, minimum width=50pt, minimum height=17pt,inner sep=0pt]
    \tikzstyle{dots}=[ellipse, fill=black!0, draw=black!60, minimum width=50pt, minimum height=17pt,inner sep=0pt]
	\tikzstyle{tapped neuron}=[ellipse, fill=black!5, draw=black!60, minimum width=50pt, minimum height=17pt, thick]
	\pgfsetshortenstart{0.8pt}
	\pgfsetshortenend{0.8pt}

   	\node[] (L1) at (0.4*\horizontalsep,0) {$\cdots$};
   	\node[] (L2) at (1.1*\horizontalsep,-\verticalsep) {$\cdots$};
	
	\node[neuron, label=center:$\mathbf{z}^{[l]} \bigg| \mathbf{a}^{[l]}$] (L3) at (2*\horizontalsep, 0) {};
	\node[tapped neuron, label=center:$\mathbf{z}^{[l+1]} \bigg| \mathbf{a}^{[l+1]}$] (L4) at (3*\horizontalsep, -\verticalsep) {};
	\path (L3) edge (L4);
	\path (L2) edge (L3);
	\path (L1) edge (L3);
	
	\node[neuron, label=center:$\mathbf{z}^{[l+2]} \bigg| \mathbf{a}^{[l+2]}$] (L5) at (4*\horizontalsep, 0) {};
	\node[tapped neuron, label=center:$\mathbf{z}^{[l+3]} \bigg| \mathbf{a}^{[l+3]}$] (L6) at (5*\horizontalsep, -\verticalsep) {};
	\path (L5) edge (L6);
	\path (L4) edge (L5);
	\path (L3) edge (L5);
	
	\node[neuron, label=center:$\mathbf{z}^{[l+4]} \bigg| \mathbf{a}^{[l+4]}$] (L7) at (6*\horizontalsep, 0) {};
	\node[tapped neuron, label=center:$\mathbf{z}^{[l+5]} \bigg| \mathbf{a}^{[l+5]}$] (L8) at (7*\horizontalsep, -\verticalsep) {};
	\path (L7) edge (L8);
	\path (L6) edge (L7);
	\path (L5) edge (L7);

    \node[] (L9) at (7.7*\horizontalsep,0) {$\cdots$};
    
	\path (L7) edge (L9);
	\path (L8) edge (L9);
	
\end{tikzpicture}
\end{scriptsize}
        \caption{Figure for the derivation of equation (\ref{eq:chainRule2}). The top row of nodes are hidden neurons. The bottom row of nodes are output neurons. Inputs are not shown.}
        \label{fig:chainRuleDerivation2}
    \end{figure}
    Consider an $L$-layered ForecastNet with a single hidden neuron in the hidden cell and a single linear output neuron as illustrated in \figurename{~\ref{fig:chainRuleDerivation2}} (the inputs are not shown as they do not contribute to the error backpropagated through hidden cells). For some layer $l$ in this network, $W^{[l]}$ is the weight matrix, $\bar{b}^{[l]}$ is the bias vector, $\mathbf{a}^{[l]}$ is the output vector, and $\mathbf{z}^{[l]}=W^{[l]T} \mathbf{a}^{[l-1]} + \bar{b}^{[l]}$. Using the chain rule of calculus, the derivative of the loss function $\mathcal{L}$ with respect to the weights $W^{[l]}$ at layer $l$ is given by
    \begin{align*}
    \frac{\partial \mathcal{L}}{\partial W^{[l]}} 
    %&= \frac{\partial \mathcal{L}}{\partial \mathbf{z}^{[l]}} \frac{\partial \mathbf{z}^{[l]}}{\partial W^{[l]}} \nonumber \\
    &= 	\frac{\partial \mathcal{L}}{\partial \mathbf{a}^{[l]}} 
    \frac{\partial \mathbf{a}^{[l]}}{\partial W^{[l]}}
    \end{align*}
    Hidden layer $l$ links to layers $l+1$ and $l+2$. Thus, the derivative with respect to $\mathbf{a}^{[l]}$ is expanded as follows
    \begin{align*}
    \frac{\partial \mathcal{L}}{\partial W^{[l]}} 
    &= \left(
    \frac{\partial \mathcal{L}}{\partial \mathbf{z}^{[l+1]}} \frac{\partial \mathbf{z}^{[l+1]}}{\partial \mathbf{a}^{[l]}} +
    \frac{\partial \mathcal{L}}{\partial \mathbf{z}^{[l+2]}} \frac{\partial \mathbf{z}^{[l+2]}}{\partial \mathbf{a}^{[l]}}
    \right)
    \frac{\partial \mathbf{a}^{[l]}}{\partial W^{[l]}} \\
    &= 	
    \frac{\partial \mathcal{L}}{\partial \mathbf{z}^{[l+1]}} 
    \frac{\partial \mathbf{z}^{[l+1]}}{\partial \mathbf{a}^{[l]}}
    \frac{\partial \mathbf{a}^{[l]}}{\partial W^{[l]}}
    +
    \frac{\partial \mathcal{L}}{\partial \mathbf{z}^{[l+2]}} 
    \frac{\partial \mathbf{z}^{[l+2]}}{\partial \mathbf{a}^{[l]}}
    \frac{\partial \mathbf{a}^{[l]}}{\partial W^{[l]}}
    \end{align*}
    Layer $l+2$ links to layers $l+3$ and $l+4$ and $\nicefrac{\partial \mathcal{L}}{\partial \mathbf{a}^{[l+2]}}$ can be expanded in a similar manner to the above. This expansion process is continued until the final output layer $L$ is reached, which produces (\ref{eq:chainRule2}).
\end{proof}

Similarly, consider a chain of neurons in a standard $L$-layered feed-forward neural network with loss function $\mathcal{L}$. Repeated application of the chain rule of calculus results in the following product
\begin{align}
\label{eq:chainRule1}
\frac{\partial \mathcal{L}}{\partial W^{[l]}} 
= 
&\frac{\partial \mathcal{L}}{\partial \mathbf{a}^{[L]}} 
\frac{\partial \mathbf{a}^{[L]}}{\partial \mathbf{z}^{[L]}}
\frac{\partial \mathbf{z}^{[L]}}{\partial \mathbf{a}^{[L-1]}}
\frac{\partial \mathbf{a}^{[L-1]}}{\partial \mathbf{z}^{[L-1]}}
\cdots \nonumber\\
%\frac{\partial \mathbf{z}^{[l+3]}}{\partial \mathbf{a}^{[l+2]}}
%\frac{\partial \mathbf{a}^{[l+2]}}{\partial \mathbf{z}^{[l+2]}}
&\frac{\partial \mathbf{z}^{[l+2]}}{\partial \mathbf{a}^{[l+1]}}
\frac{\partial \mathbf{a}^{[l+1]}}{\partial \mathbf{z}^{[l+1]}}
\frac{\partial \mathbf{z}^{[l+1]}}{\partial \mathbf{a}^{[l]}}
\frac{\partial \mathbf{a}^{[l]}}{\partial \mathbf{z}^{[l]}}
\frac{\partial \mathbf{z}^{[l]}}{\partial W^{[l]}}
\end{align}
This equation can be expressed in the form $\lambda^L$ \citep{Caterini2018Deep, goodfellow2016deep}. If $|\lambda|<1$, the term vanishes towards $0$ as $L$ grows. If $|\lambda|>1$, the term explodes as $L$ grows.

Equation (\ref{eq:chainRule1}) contains a product of factors whereas (\ref{eq:chainRule2}) contains a sum of terms. A sum of terms with values less than one does not tend to zero (vanish) as a product of factors with values less than one would.

The interleaved outputs mitigate, but do not eliminate the vanishing gradient problem. The term $\Psi_k$ is a product of derivatives formed by the chain rule of calculus. The number of factors in this product grows proportional to $k$. For a distant layer where $k$ is large, $\Psi_k$ will have many factors. The result is that gradients back-propagated from distant layers are still susceptible to vanishing gradient problems. However, for nearby outputs where $k$ is small, $\Psi_k$ will have few factors and so gradients from \textit{nearby} outputs are less likely to experience vanishing gradient problems. Thus, nearby outputs can provide guidance to local parameters during training, resulting in improved convergence as the effective depth of the network is reduced.

    %\input{backpropagation}
    %\input{vanishingGrad}
    %Joel Dabrowski
%Tapped Deep Network
%Section: Methodology

\section{Material and Methods}%Approach

\subsection{Datasets}

A set of models are compared on a synthetic dataset and nine real-world datasets sourced from various domains. These include weather, environmental, energy, aquaculture, and meteorological domains. Datasets with seasonal components and complex trends are hand-picked to ensure that they provide a sufficiently challenging problem. Properties such as varying seasonal amplitude, varying seasonal shape, and noise were sought. For example, the shape of the seasonal cycle is non-stationary and changes over time in most datasets. Additionally, properties such as seasonality assist in demonstrating the ability for models to learn long-term dependencies in the data.

The synthetic dataset is used to provide a baseline. The data is generated according to
$x_t = 2 \sin \left( 2 \pi f t \right ) + \nicefrac{1}{3} \sin \left (\nicefrac{2 \pi f t}{5} \right)$, 
%\begin{align}
%x_t = 2 \sin \left( 2 \pi f t \right ) + \frac{1}{3} \sin \left (2 \pi \frac{f}{5} t \right)
%\end{align}
where $f$ is the frequency and $t$ denotes time. The low frequency sinusoid emulates a long-term time-varying trend, whereas the high frequency sinusoid emulates seasonality. Models are expected to perform well on this dataset because it contains no noise. The properties of all the datasets are summarised in Table \ref{table:datasetProperties}. In figures and tables the datasets are referred to by their abbreviations.

%Four freely available real-world datasets from Rob Hyndman's well-known \textit{Time Series Data Library} (TSDL) \citep{datamarket2018} were selected with the required properties. The first dataset contains monthly England temperature readings spanning 1723 to 1970. The second dataset contains monthly river flow readings of the Chang Jiang river at Hankou, China, spanning 1865 to 1979. The third dataset contains ozone readings spanning 1932 to 1972 over Arosa. The fourth dataset contains monthly Lake Erie level readings from 1921 to 1970. 
%
%An energy dataset containing hourly readings of power consumption in New South Wales, Australia over the period of 2010 and 2011 \citep{aemo2019} is selected. Three aquaculture prawn pond water quality variable datasets \cite{dabrowski2018state} containing sensor readings of dissolved oxygen, pH, and water temperature are used. These datasets were originally sampled at 15 minute intervals, but were resampled to hourly intervals by averaging over the hour. Finally, the Metro Interstate Traffic Volume Dataset from the UCI machine learning repository \cite{uci2019} is selected.

%
% Table data from printouts of datasetTable.py
\begin{table*}[t]
    \centering
    \caption{Dataset properties.}
    \label{table:datasetProperties}
    \begin{center}
        \begin{scriptsize}
            \begin{tabular}{l l c c c c c c c}
                \toprule
                Dataset & Abbreviation & Resolution & Period & Length & Minimum & Maximum & Mean & Std. Dev. \\
                \midrule
                Synthetic                                  & Synth.	& -         & 20 & 4320 & -2.33 & 2.33 & -0.00 & 1.43 \\
                England temperature \cite{datamarket2018}  & Weath. & Monthly   & 12 & 3261 & 0.10 & 18.80 & 9.27 & 4.75 \\
                River flow \cite{datamarket2018}           & River  & Monthly 	& 12 & 1492 & 3290 & 66500 & 23157.60 & 13087.40 \\
                Electricity \cite{aemo2019}                & Elect. & Hourly 	& 24 & 19224 & 5514 & 14580 & 8709.79 & 1360.29 \\
                Traffic Volume \cite{uci2019}              & Traff.	& Hourly 	& 24 & 8776 & 125 & 7217 & 3269.26 & 2021.57 \\
                Lake levels \cite{datamarket2018}          & Lake	& Monthly 	& 12 & 648 & 10 & 20 & 15.08 & 2.00 \\
                Dissolved Oxygen \cite{dabrowski2018state} & DO	    & Hourly 	& 24 & 2422 & 5.66 & 7.94 & 6.50 & 0.53 \\
                pH \cite{dabrowski2018state}               & pH	    & Hourly 	& 24 & 2422 & 8.07 & 11.15 & 8.56 & 0.21 \\
                Pond temperature \cite{dabrowski2018state} & Temp.	& Hourly 	& 24 & 2422 & 24.38 & 31.97 & 27.74 & 1.85 \\
                Ozone \cite{datamarket2018}                & Ozone	& Monthly 	& 12 & 516 & 266 & 430 & 338.00 & 38.30 \\
                \bottomrule
            \end{tabular}
        \end{scriptsize}
    \end{center}
\end{table*}

\subsection{Models}
\label{sec:models}

Four deep-learning based benchmark models are compared to four variations of ForecastNet. These models include deepAR \citep{flunkert2017deepar}, the TCN \citep{bai2018empirical}, the sequence-to-sequence (encoder-decoder) model \citep{Sutskever2014Sequence}, and the attention model \citep{bahdanau2014neural}. For completeness, a single layer MLP, a free-form seasonal Dynamic Linear Model (DLM) \citep{west2013bayesian}, and a seasonal autoregressive moving average (SARIMA) model are included in the comparison. The DLM (a state space model) and the SARIMA model are well-known statistical models that are used for time-series forecasting \citep{hyndman2018forecasting}.  

%%
%\begin{figure}[t]{}
%	\centering
%	\input{figures/fig_FN}
%	\caption{ForecastNet with a densely connected hidden cell structure as used in FN and FN2. The number of hidden neurons in each layer is denoted by h. See together with \figurename{~\ref{fig:forecastNet}}.}
%	\label{fig:FN}
%\end{figure}
%%
%%\captionsetup{skip=0pt}
%\begin{figure}[t]{}
%	\centering
%	\input{figures/fig_cFN}
%	\caption{ForecastNet with a CNN hidden cell structure as used in cFN and cFN2. The number of filters, kernel size, padding, strides, and number of hidden neurons are denoted by f, k, p, s, and h respectively. See together with \figurename{~\ref{fig:forecastNet}}.} 
%	\label{fig:cFN}
%\end{figure}
%%
Four variations of ForecastNet are tested\footnote{Note that we tested the use of self-attention mechanisms in the hidden cells. We were unable to achieve any significant performance improvement compared with FN.}:
\begin{description}[font=\normalfont, noitemsep, topsep=0pt]
	\item[FN:] Densely connected hidden layers in each hidden cell and a Gaussian mixture density output layer.
	\item[cFN:] CNNs in the hidden cell and a Gaussian mixture density output layer.
	\item[FN2:] This is identical to FN, but with a linear output layer instead of a mixture density output layer.
	\item[cFN2:] This is identical to cFN, but with a linear output layer instead of a mixture density output layer.
\end{description}

The set of models are tested on a datasets that have a seasonal component with period denoted by $\tau$. The number of inputs in all models is set to $2\tau$ and the number of outputs (forecast-steps) is set to $\tau$. Thus, the models are trained to forecast one seasonal cycle ahead in time, given the two previous cycles of data. 

To avoid possible bias between the models, they are configured to be as similar as possible. This is achieved by limiting the number of neurons in the models to similar values. Configuration details are provided in Table \ref{table:modelConfig}. In figures and tables, the sequence-to-sequence model is denoted by `seq2seq' and the attention model is denoted by `Att'. % in \ref{sec:additionalPlots}.
\begin{table*}[!t]
	\centering
	\caption{Model configuration. $\tau$ is the seasonal period of the dataset and ReLU is the rectified linear unit.}
	\label{table:modelConfig}
	\begin{center}
		\begin{scriptsize}
			%			\begin{tabular}{l p{0.8\columnwidth}}
			\begin{tabular}{l p{0.9\textwidth}}
				\toprule
				Model & Configuration \\
				\midrule
				FN			
				& Each hidden cell comprises two densely connected hidden layers, each with 24 ReLU neurons. The rectified linear unit (ReLU) with He initialisation \citep{he2015delving} is used as the activation function. \\
				FN2
				& Identical to FN2 but with a linear output layer instead of a mixture density layer. \\
				cFN 	
				& Each hidden cell comprises a convolutional layer with 24 filters, each with a kernel size of 2, followed by a average pooling layer with a pool size of 2 and stride of 1. The convolutional and pooling layers are duplicated and followed by a dense layer with 24 ReLU neurons. \\
				cFN2
				& Identical to cFN2 but with a linear output layer instead of a mixture density layer. \\
				DeepAR		
				& The sequence-to-sequence architecture with single layered LSTMs are used in the encoder and decoder. The mixture density output of the network is a Gaussian (normal) distribution.\\
				TCN
				& The TCN contains a convolutional layer with 32 filters, each with a kernel size of 2 for the Synthetic, Weather, Elect., and River datasets. For the remaining datasets, the TCN contains a convolutional layer with 64 filters, each with a kernel size of 3. The output contains a dense layer with $\tau$ linear units. \\
				Attention
				& Encoder has a bidirectional LSTM and the decoder has a single layered LSTM. The LSTM cells are configured with 24 ReLU units.\\
				Seq2Seq		
				& Encoder and decoder use a single layered LSTM. The LSTM cells are configured with 24 RelU units. \\
				MLP			
				& Feed forward MLP with a single hidden layer comprising $4\tau$ ReLU hidden neurons. A set of $2\tau$ inputs are provided and a set of $\tau$ linear outputs are used. \\
				DLM
				& The DLM used is the free-form seasonal model with a zero order trend component \citep{west2013bayesian}. Model fitting is performed using a modified Kalman filter \citep{pydlm} \\
				SARIMA
				& Standard form SARIMA(p,d,q)(P,D,Q)s with: %\newline 
				\textbf{Synthetic}: SARIMA(1,1,1)(1,1,0)20,
				\textbf{Weather}: SARIMA(2,0,3)(0,1,0)12,
				\textbf{Elect.}: SARIMA(3,1,4)(0,1,0)24, %\newline 
				\textbf{River}: SARIMA(2,0,4)(0,1,0)12, 
                \textbf{Traff.}: SARIMA(3,1,1)(0,1,0)24,
                \textbf{Lake}: SARIMA(2,0,8)(0,1,0)12, %\newline 
                \textbf{DO}: SARIMA(2,0,6)(0,1,0)24,
                \textbf{pH}: SARIMA(4,1,3)(0,1,0)24,
                \textbf{Temp.}: SARIMA(2,1,4)(0,1,0)24, and %\newline 
                \textbf{Ozone}: SARIMA(3,0,4)(0,1,0)12, \\
				\bottomrule
			\end{tabular}
		\end{scriptsize}
	\end{center}
\end{table*}

\subsection{Training and Testing}
\label{sec:training}

The datasets are scaled to the range of $[0, 1]$ for training. Each dataset is split into a training and a test set, where the last 10\% of the data are used for the test set. The training and test sets both comprise a long sequence of values. These sequences are converted into a set of samples that the models can process. A sample is extracted using a sliding window of length $3\tau$. The first $2\tau$ samples in this window form the input sequence to the model. The last $\tau$ samples form the forecast target values. The sliding window is slid across the dataset sequence to produce a set of samples. The set of training samples are shuffled prior to training. 
%The last sample in the training set is used as a validation set when necessary.
% During testing the test samples are extracted using a $3\tau$ length window that is shifted sample-by-sample over the test set. For a test set of size $N$, this procedure produces $N-3\tau+1$ forecasts. 

The ADAM algorithm \citep{kingma2014adam} is used to minimise the mean squared error in all machine learning models. The learning rate is searched over the range $10^{-i}, i \in [2,\dots, 6]$. Early stopping is used to address overfitting and defines the number of epochs. The Mean Absolute Scaled Error (MASE) \citep{Hyndman2006Another} is used to evaluate performance of the models. For completeness, results with the Symmetric Mean Absolute Percentage Error (SMAPE) \cite{hyndman2018forecasting} are additionally provided.

    %Joel Dabrowski
%Tapped Deep Network

\section{Results and Discussion}

%\subsection{Forecasting Performance}
%\label{sec:forecastingPerformance}

\subsection{Time-Invariance Test}

%
% Figure generated by synthetic2DatasetTest.py
\begin{figure}[t]{}
    \centering
    \includegraphics{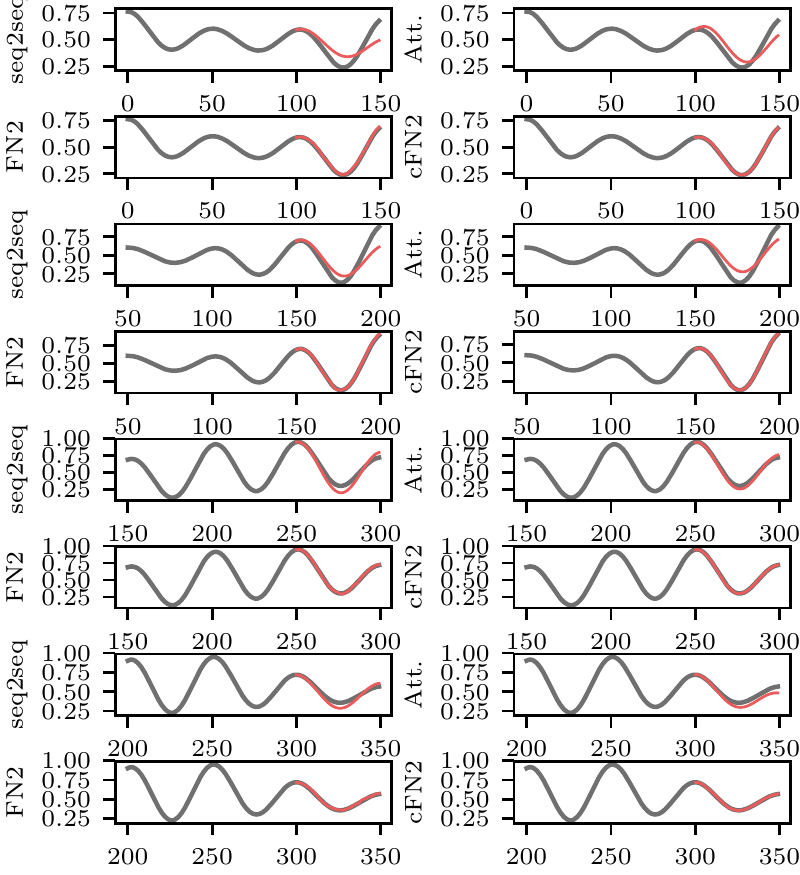}
    \caption{Synthetic dataset forecasts for for the seq2seq, Att., FN2, and cFN2 models at starting time indices 0, 50, 150 and 200.}
    \label{fig:synthetic2_plots}
\end{figure}

Before comparing models on the datasets, the difference between time-invariant and time-variant models is demonstrated using a variation of the synthetic dataset given by
\begin{align*}
x_t = 
\frac{1}{2} \sin \left( \frac{2 \pi f}{6} t \right ) 
\left( 
\frac{3}{5} \sin \left( 2 \pi f t \right ) + 
\frac{1}{5} \sin \left (\frac{2 \pi f}{5} t \right)
\right)
\end{align*}
%The sinusoid with an angular frequency of $2 \pi f$ emulates a seasonal cycle, the sinusoid with the slow angular frequency of $\nicefrac{2 \pi f}{5}$ emulates a long-term time-varying trend, and the sinusoid with an angular frequency of $\nicefrac{2 \pi f}{3}$ performs amplitude modulation. 
%The (extra) first sinusoid performs amplitude modulation, the second sinusoid emulates a seasonal cycle, and the third sinusoid emulates a time-varying trend.
The second and third sinusoids emulate a seasonal cycle time-varying trend as previously presented. The first sinusoid is included to perform amplitude modulation.

The amplitude modulation repeats every 6 cycles of the seasonal period. A model is only presented with two seasonal cycles as inputs. The model thus never observes a complete pattern, which is the full cycle of the amplitude modulation. A time-variant model's parameters are able to change over time, which enables the model to adapt to the variations in amplitude. A time-invariant model is less agile and is expected to struggle with large amplitude variations.

The seq2seq, attention, FN2, and cFN2 models are trained on the dataset. Results of forecasts from inputs starting at time indices 0, 50, 150 and 200 are presented in \figurename{~\ref{fig:synthetic2_plots}}. As expected, the time-invariant seq2seq model struggles maintaining accurate forecasts when there are large variations in the amplitude. The attention model performs better due to the time-variant attention mechanism. However, the time-variant ForecastNet models adapt well to the large variations in the signal amplitude. We argue that this is due to the time-variant properties of ForecastNet.

\subsection{Model Comparison Error Results}

% Table generated by tableGeneration.py
\begin{table*}[t]{}
    \caption{Average MASE and SMAPE (in brackets) of the models results over the test datasets. The last row indicates the sum of Borda counts of the models over the datasets (a higher value indicates more points in the voting score). Boldface numbers highlight top results.}
    \label{table:mapeResults}
    \setlength\tabcolsep{4.5pt}
    \begin{center}
        \begin{scriptsize}
            \begin{tabular}{l c c c c c c c c c c c}
                \toprule
                & FN & cFN & FN2 & cFN2 & deepAR & Seq2Seq & Attention & TCN & MLP & DLM & SARIMA \\
                \midrule
                Synth. & 0.00 (1.5) & 0.01 (1.7) & \textbf{0.00 (1.3)} & 0.00 (1.4) & 0.03 (2.3) & 0.01 (1.7) & 0.04 (2.8) & 0.05 (3.0) & 0.01 (1.6) & 0.64 (23.0) & 0.29 (11.7) \\
                Weath. & 0.46 (17.1) & 0.40 (15.3) & 0.46 (16.7) & \textbf{0.31 (12.3)} & 0.46 (17.3) & 0.43 (16.0) & 0.37 (14.3) & 0.47 (17.3) & 0.47 (17.1) & 0.52 (18.9) & 0.61 (22.0) \\
                Elect. & 1.12 (12.5) & 1.04 (11.9) & 0.89 (10.9) & \textbf{0.54 (6.7)} & 1.77 (18.2) & 1.00 (11.7) & 1.39 (14.7) & 1.09 (12.9) & 1.34 (15.4) & 1.73 (19.1) & 1.26 (15.1) \\
                River & 0.71 (24.1) & 0.66 (23.6) & 0.66 (25.2) & \textbf{0.39 (15.4)} & 0.86 (28.7) & 0.57 (21.9) & 0.53 (19.4) & 0.85 (30.5) & 0.85 (30.3) & 0.77 (26.8) & 0.87 (29.5) \\
                Traff. & 2.23 (46.0) & 1.95 (40.3) & 1.44 (32.5) & \textbf{0.82 (21.2)} & 2.01 (40.3) & 1.78 (36.6) & 1.94 (39.3) & 2.20 (44.4) & 2.36 (47.8) & 2.40 (43.5) & 2.32 (42.9) \\
                Lake & \textbf{1.42 (10.1)} & 1.61 (12.0) & 1.58 (12.3) & 1.69 (13.0) & 1.61 (12.3) & 1.73 (13.4) & 1.56 (11.2) & 2.03 (13.8) & 1.57 (12.1) & 1.95 (15.5) & 1.69 (12.1) \\
                DO & \textbf{0.54 (6.1)} & 0.62 (7.2) & 0.62 (6.9) & 0.54 (6.2) & 0.71 (8.2) & 0.73 (8.6) & 2.11 (26.5) & 0.78 (9.1) & 0.77 (8.6) & 0.77 (8.7) & 0.64 (7.0) \\
                pH & 1.41 (13.6) & 1.23 (12.0) & 1.26 (12.4) & \textbf{1.01 (9.6)} & 2.64 (25.1) & 1.70 (16.6) & 1.42 (14.4) & 1.35 (13.1) & 1.90 (18.3) & 1.29 (12.7) & 1.68 (16.5) \\
                Temp. & 1.90 (9.0) & 1.90 (10.0) & 1.66 (8.5) & 2.00 (10.7) & 1.95 (9.8) & 2.13 (10.6) & 2.23 (11.4) & 3.18 (14.2) & 2.10 (9.7) & 3.67 (18.4) & \textbf{1.64 (7.0)} \\
                Ozone & 0.72 (33.9) & 0.78 (36.2) & 0.69 (33.9) & 0.79 (36.0) & 1.03 (44.7) & 1.50 (61.8) & \textbf{0.58 (28.6)} & 0.69 (33.5) & 0.66 (32.5) & 0.76 (30.6) & 0.89 (36.5) \\
                \midrule
                Borda Count & 74 & 76 & \textbf{90} & \textbf{90} & 43 & 57 & 65 & 42 & 51 & 31 & 41 \\
                \bottomrule
            \end{tabular}
        \end{scriptsize}
    \end{center}
\end{table*}

The average MASE and SMAPE over all forecasts on each dataset's test set is provided in Table \ref{table:mapeResults}. ForecastNet produces the best results on 8 of the 10 datasets. The cFN2 variation of ForecastNet achieves the best results on 4 of these 8 datasets. This result is reinforced with Borda counts provided in the last row. A Borda count ranks a set of $M$ models with integers $(1, \dots , M)$ such that the model with the lowest MASE is assigned a value of $M$ (a higher vote) and the model with the highest MASE is assigned a value of 1 (a lower vote). Borda counts thus provide a more \textit{aggregated} evaluation. FN2 and cFN2 are voted as the best models with the highest Borda counts. These are followed by cFN, FN and the attention models respectively.

The attention model produced the lowest error for the ozone dataset. The attention model is a relatively complex model and its attention mechanism assists in modelling long-term dependencies. FN2 provides strong competition to the attention model over the remaining datasets. This is despite it having an arguably a simpler architecture with no gating structures to reduce vanishing gradients. We argue that a key reason why ForecastNet performs so well is that it is not a time-invariant model.

Increasing the model complexity generally resulted in improved forecast performance in this study. For example, cFN generally outperforms FN. However, simpler models do not fail on the datasets. For example, the MLP provided comparably accurate forecasts despite its simplicity.

As expected, the DLM and SARIMA  statistical models performed well despite being linear models \citep{Makridakis2018M4}. For example, the SARIMA model achieved the lowest error on the pond temperature dataset. This suggests that the dynamics of this dataset are more linear. However, with the non-linear trends, amplitudes, and cyclic shapes in the other datasets, the DLM and SARIMA models did not perform as well the non-linear neural network-based models. It has however been shown that such linear statistical models can outperform complex machine learning models when the sample size is small \citep{cerqueira2019machine}.

Of the deep neural network-based models, the TCN performed the worst on several datasets. \citet{bai2018empirical} suggest that the model is in a simplified form and improved results may be possible by using a more advanced architecture. Furthermore, the TCN is designed to perform dilated convolutions over many samples. Of the datasets used in this study, the maximum number of input samples was 48 for datasets with a period of 24 hours. This may be too few to demonstrate the effectiveness of the TCN.

Several MASE results are above unity. In multi-step-ahead forecasting, this does not necessarily imply that a forecasting model produces results worse than the Naive model. In the MASE calculation, Naive is produces one-step-ahead forecasts with \textit{in-sample} data provided for each previous step. The forecasting model is computed \textit{out-of-sample}. That is, Naive has access to the ground-truth data in the MASE calculation, whereas the forecasting model does not.

\subsection{Model Comparison Box-Whisker Plots}
%
% Figure generated by boxPlotsCrossValidation.py
\begin{figure}[t]{}
	\centering
	\includegraphics{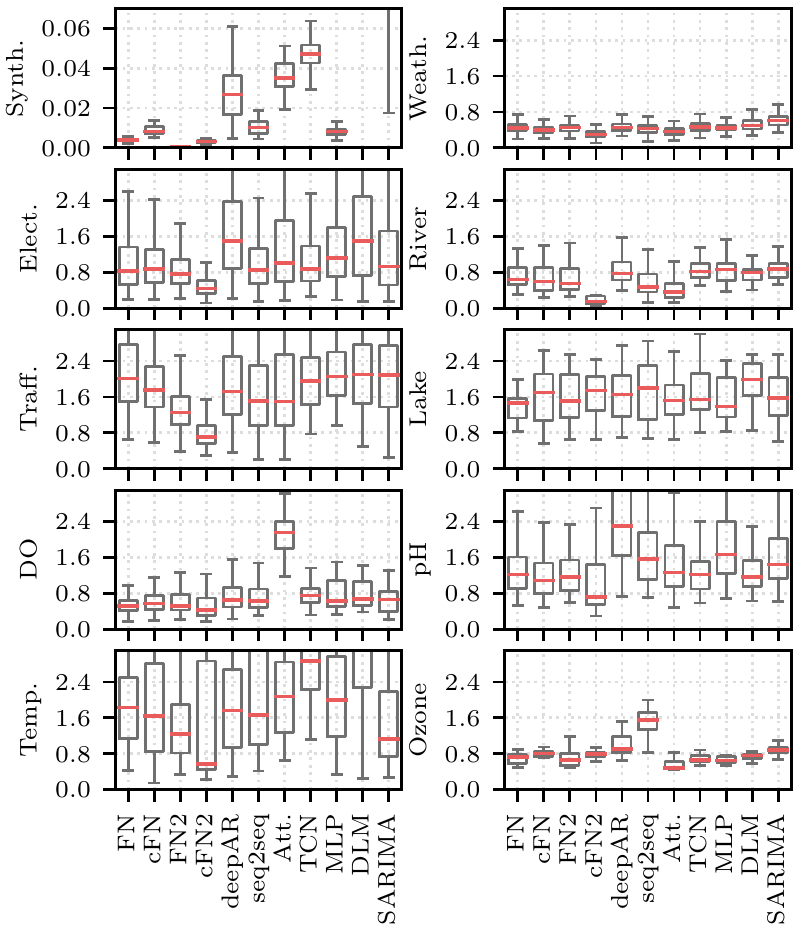}
	\caption{Box plot of the MASE over the set of forecasts produced for each training dataset. The DLM and SARIMA  boxes are outside of the plot range for the synthetic dataset.}
	\label{fig:boxplot}
\end{figure}
Box-whisker plots of the results over all the forecasts in each dataset's test set are provided in \figurename{~\ref{fig:boxplot}}. ForecastNet consistently produced small boxes with low median values. The small boxes indicate that there is little variation in the accuracy over the set of forecasts. This indicates some form of robustness in the ForecastNet model. The low median values indicate a high level of accuracy.

There was significant variation over the different models in the synthetic dataset box-whisker plots. For this dataset, the densely connected networks such as FN, FN2 and MLP have small boxes. DeepAR had a large box which indicates higher variation in the forecasts. This indicates that DeepAR is less robust for this dataset.

The models generally perform well on the weather dataset. This may be due to a more consistent seasonal amplitude in this dataset compared with the other real-world datasets. The lake and pH datasets have varying trends, amplitudes, and seasonal shape resulting in higher errors than other datasets. ForecastNet and the attention model seem to model these variations better given their lower errors.

In datasets such as electricity, traffic, and pH, ForecastNet produced low errors with small boxes indicating reliable and accurate forecasts. Especially in the electricity and traffic datasets, it is evident that increased model complexity and removing the mixture density output results in lower errors and a more robust model. The mixture density outputs can reduce accuracy because the learning algorithm seeks to \textit{simultaneously} minimise two variables in the normal distribution's log likelihood function. This is opposed to minimising a single variable in mean squared error loss function used for a linear output layer.

    %Joel Dabrowski
%ForecastNet
%Section: Conclusion

\section{Summary and Conclusion}

In this study, ForecastNet is proposed as novel deep neural architecture for multi-step-ahead time series forecasting. Its architecture breaks from convention of structuring a model around the RNN or CNN. The result is a model that is time-variant compared with the RNN and CNN, which are time-invariant.

We provide comparison over seven state-of-the-art deep learning and statistical models for forecasting seasonal time series data. The comparison is performed on 10 seasonal time-series datasets selected from various domains. We demonstrate that ForecastNet is both accurate and robust on all datasets. It outperforms other models in terms of MASE and SMAPE on 8 of the 10 datasets and is ranked as the best performing model overall with Borda counts.

In future work, shortcut-connections within and between hidden cells could be investigated. Furthermore, more work into integrating self-attention into the model will provide benefits relating to model interpretability. Lastly, by avoiding parameter sharing to achieve a time-variant model, ForecastNet can require more memory. An investigation into using memory reduction techniques (such as quantization) could be explored.

    \bibliography{bibliography}
    \bibliographystyle{icml2020}
    
    \ifgenAppendices
        \newpage
        \appendix
%        \setcounter{page}{1}
        %Joel Dabrowski
%ForecastNet
%Section: Appendix

\section{Additional Results}
\label{sec:additionalResults}

\subsection{Forecast Plots}

% Figure generated by plotForecasts.py
\begin{figure}[!b]{}
    \centering
    \includegraphics{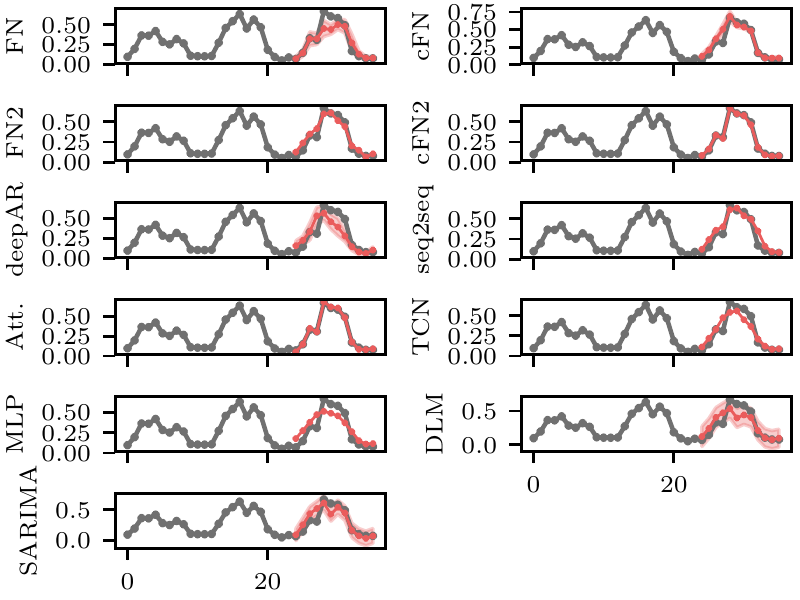}
    \caption{Plots of the forecasts for each model on the river flow dataset. The grey curve plots the sensor data and the red curve plots the forecasts. The standard deviation for the relevant models is plotted as the red shaded region.}
    \label{fig:forecast_riverflow}
\end{figure}

To illustrate forecasting ability, a plot of the forecasts for each model on the river flow dataset is presented in \figurename{~\ref{fig:forecast_riverflow}}. The uncertainty (standard deviation) is plotted for the FN, cFN, deepAR, DLM, and SARIMA models. 

As illustrated in \figurename{~\ref{fig:forecast_riverflow}}, the river flow data has a varying amplitude and an irregular seasonal curve. The first peak has a lower amplitude than the second and third peak. Furthermore, the shape of each the seasonal cycle is unique. An accurate model is one that is able to adapt to these changes by learning the underlying dynamics which generate these changes.

Despite the variations over seasonal cycles, cFN2 provided highly accurate forecasts, where fine intricacies in the data were forecast. In this example, the attention model also produces an accurate forecast, however cFN2 outperforms the attention model on this dataset. The DeepAR and the sequence-to-sequence models provided more smooth forecasts. DeepAR has the advantage of providing uncertainty with the forecast.

The DLM and SARIMA models forecast a curve that is similar in form to the curve from the previous cycle. This is expected as both of these models consider one cycle of data in the past to compute the forecast. Furthermore, these statistical models assume a linear trend, whereas the actual trend is nonlinear and appears more stochastic in nature. The nonlinear models are more capable of modelling the nonlinear dynamics.

\subsection{Computational Complexity}
\label{sec:computationalComplexity}

%
% Table data from printouts of computationalComplexity.py
\begin{table*}[!t]{}
    \caption{Average epoch duration over 10 epochs. The times are represented as a multiple of the epoch duration of the MLP for each respective dataset.}
    \label{table:epochTime}
    \setlength\tabcolsep{4pt}
    \begin{center}
        \begin{scriptsize}
%            \rotatebox{90}{
                \begin{tabular}{l c c c c c c c c c}
                    \toprule
                    & FN & cFN & FN2 & cFN2 & deepAR & Seq2Seq & Attention & TCN & MLP \\
                    \midrule
                    Synth.  & 12.62 & 36.84 & 5.21 & 32.48 & 39.38 & 36.68 & 69.60 & 5.04 & 1.00 \\
                    Weath.  & 6.55 & 18.88 & 3.38 & 18.22 & 23.39 & 22.14 & 36.93 & 2.28 & 1.00 \\
                    Elect.  & 15.07 & 77.90 & 6.09 & 74.79 & 48.26 & 45.60 & 101.55 & 13.96 & 1.00 \\
                    River  & 7.05 & 18.63 & 3.67 & 17.52 & 22.61 & 21.18 & 35.83 & 2.85 & 1.00 \\
                    Traff.  & 14.06 & 60.68 & 5.56 & 56.94 & 44.38 & 41.54 & 86.20 & 15.86 & 1.00 \\
                    Lake    & 7.52 & 18.41 & 3.68 & 16.42 & 23.65 & 22.52 & 37.64 & 2.62 & 1.00 \\
                    DO      & 17.41 & 46.64 & 6.61 & 39.28 & 47.29 & 44.70 & 88.71 & 8.96 & 1.00 \\
                    pH      & 17.44 & 45.24 & 6.69 & 42.48 & 46.21 & 44.79 & 89.40 & 9.08 & 1.00 \\
                    Temp. & 16.97 & 48.62 & 6.57 & 41.41 & 47.27 & 44.39 & 88.32 & 8.94 & 1.00 \\
                    Ozone   & 7.89 & 17.45 & 3.91 & 15.68 & 23.50 & 21.93 & 36.53 & 2.87 & 1.00 \\
                    \midrule
                    Median & 13.34 & 41.04 & 5.38 & 35.88 & 41.88 & 39.11 & 77.90 & 6.99 & 1.00 \\
                    \bottomrule
                \end{tabular}
%            }
        \end{scriptsize}
    \end{center}
\end{table*}

To demonstrate the computational complexity of the models, their training times were considered. Each model was trained on each dataset for 10 epochs and these epoch times were logged. The results are presented in Table \ref{table:epochTime}.
To provide some independence from the platform on which the models were trained, the results are presented as an average epoch time of the particular model divided by the average epoch time of the MLP. The results are thus represented as a multiple of the epoch duration of the MLP.
%Note that the TCN was implemented on a different platform than the other models.%For reference, TensorFlow's number of trainable parameters are provided in the bottom portion of the table. %Table \ref{table:parameters}. 

The MLP had the simplest architecture and thus provided the shortest epoch times. Second to the MLP are the FN2 and the TCN. These models required 5 to 7 times the amount of time to train an epoch compared to the MLP. Note that FN2 is ranked as a top model and also has the second lowest median computation time, resulting in a highly attractive model. 

Using a density mixture output on ForecastNet significantly increased the epoch time. This is evident when the duration of FN is compared to FN2 and the duration of cFN is compared to cFN2. However, even with increased computational complexity, the median duration of cFN and cFN2 was less than the sequence-to-sequence, attention, and deepAR models. The attention model in particular had a high epoch duration due to its complex architecture. The results provide empirical evidence that ForecastNet provided reduced computational complexity compared with the benchmark models.

\subsection{Vanishing Gradients}
%\balance

To demonstrate that ForecastNet is able to mitigate vanishing gradients, it was compared with a deep MLP on the synthetic dataset. The deep MLP was selected for this purpose as it has no guard against vanishing gradients. The MLP was configured with 40 inputs, 20 hidden layers, and 20 outputs. Similarly, ForecastNet was configured with 20 linear outputs and a single hidden layer in each cell. This results in a total of 20 hidden layers, 20 outputs, and 40 inputs as for the deep MLP. Thus, the primary difference between the models was that ForecastNet uses interleaved outputs, whereas the deep MLP's outputs are all located at the output layer. Both models used the sigmoid activation function with Xavier normal initialisation in hidden layers. The models were trained over 10 epochs with a learning rate of $10^{-4}$.

The absolute mean value of the weights for the first and last hidden layers are plotted in the top graph of \figurename{~\ref{fig:vanishing_gradient}}. The training losses are plotted in the bottom graph of \figurename{~\ref{fig:vanishing_gradient}}. The gradient in the first layer of the MLP remained close to zero over all 10 epochs. This indicates a vanishing gradient problem. Furthermore, as indicated in the loss plot, the model did not converge to an optimal solution. In comparison, the gradients of both layers in ForecastNet were non-zero. This indicates that ForecastNet had mitigated vanishing gradients. Furthermore, ForecastNet converged to a more optimal solution compared to the deep MLP model.
%
% Figure generated by vanishingGradient.py.
\begin{figure}[!t]
	\centering
	\includegraphics{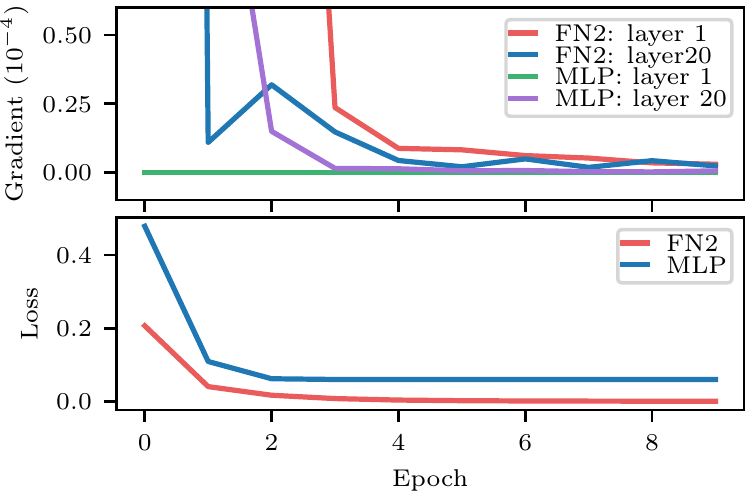}
	\caption{Gradient and loss plots for ForecastNet and a deep MLP. Layer 1 of the MLP experiences vanishing gradients. The gradient plot is cut-off at $0.6\times 10^{-4}$}
	\label{fig:vanishing_gradient}
\end{figure}

\section{Additional Model Details}
A more detailed diagram of FN and cFN are provided in \figurename{~\ref{fig:FN}} and \figurename{~\ref{fig:cFN}} respectively.
\begin{figure}[!b]{}
    \centering
    %Joel Dabrowski
%Tikz Figure
%Name: fig_forecastNetGeneral
%FN Structure

%\def\horizontalsep{1.7cm}
\def\horizontalsep{1.0cm}
\def\verticalsep{0.3cm}

\begin{scriptsize}
	\begin{tikzpicture}[->, >={Latex[length=1.5mm, width=1mm, black!40]}, draw=black!40]
	%\tikzstyle{every pin edge}=[<-,thick]
	%\tikzstyle{output}=[circle, fill=black!20, draw=black!60, minimum size=0.3cm,inner sep=0pt]
	\tikzstyle{output}=[diamond, fill=black!35, draw=black!60, minimum size=0.4cm,inner sep=0pt]
	%\tikzstyle{input}=[circle, fill=black!40, draw=black!60, minimum size=0.3cm,inner sep=0pt]
	\tikzstyle{input}=[circle, fill=black!20, draw=black!60, minimum size=0.3cm,inner sep=0pt]
	\tikzstyle{neuron}=[fill=black!0, draw=black, rotate=90, minimum width = 1.2cm, minimum height=18pt,inner sep=1pt, align=center]
	\tikzstyle{dots}=[inner sep=0pt]
	\pgfsetshortenstart{0.7pt}
	\pgfsetshortenend{0.7pt}
	
	%Inputs
	\node[input, pin={[pin edge={<-}]above:$\mathbf{x}$}] (X1) at (0*\horizontalsep, 0*\verticalsep) {};
	
	\node[dots] (dots1) at (0*\horizontalsep, -1*\verticalsep) {$\dots$};
	
	%Hidden layer 1
	\node[draw=black!15, minimum width=2.1*\horizontalsep, minimum height=1.4cm] (bdr) at 
	(1.5*\horizontalsep, -1*\verticalsep) {};
	\node[neuron] (H1a) at (1*\horizontalsep, -1*\verticalsep) {Dense \\ \textcolor{gray}{(h=24)}};
	\node[neuron] (H1b) at (2*\horizontalsep, -1*\verticalsep) {Dense \\ \textcolor{gray}{(h=24)}};
	
	%Outputs
	\node[output, pin={[pin edge={->}]below:$\hat{x}_{t+i}$}] (Y1) at (3*\horizontalsep, -2*\verticalsep) {};

	\node[dots] (dots2) at (3*\horizontalsep, -1*\verticalsep) {$\dots$};
	
	%Paths
	\path (dots1) edge (H1a);
	\path (X1) edge (H1a.north);
	\path (H1a) edge (H1b);
	\path (H1b.south) edge (Y1.west);
	\path (H1b) edge (dots2);
	%

%	% Inputs, hidden cells and outputs labels
%	\draw [-, decorate,decoration={brace, amplitude=5pt,raise=0pt},black] 
%	(3.5*\horizontalsep,1.0*\verticalsep) -- (3.5*\horizontalsep,-0.15*\verticalsep) 
%	node [black, midway, xshift=8pt, right, label={[rotate=90]center:Input}] {};
%	
%	\draw [-, decorate,decoration={brace, amplitude=5pt,raise=0pt},black] 
%	(3.5*\horizontalsep,-0.2*\verticalsep) -- (3.5*\horizontalsep,-1.8*\verticalsep) 
%	node [black,midway, xshift=12pt, right, label={[rotate=90,align=center]center:{Hidden \\ cell}}] {};
%	%	
%	\draw [-, decorate,decoration={brace, amplitude=5pt,raise=0pt},black] 
%	(3.5*\horizontalsep,-1.85*\verticalsep) -- (3.5*\horizontalsep,-3.4*\verticalsep) 
%	node [black,midway,xshift=8pt, right, label={[rotate=90]center:Output}] {};
	
	\end{tikzpicture}
\end{scriptsize}
    \caption{ForecastNet with a densely connected hidden cell structure as used in FN and FN2. The number of hidden neurons in each layer is denoted by h. }%See together with \figurename{~\ref{fig:forecastNet}}.}
    \label{fig:FN}
\end{figure}
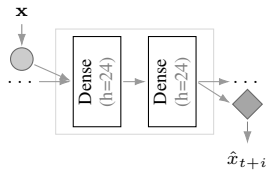
%
%\captionsetup{skip=0pt}
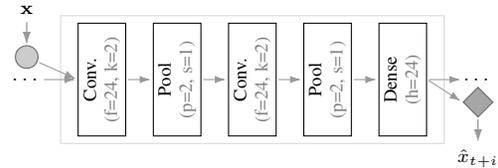
\begin{figure}[!b]{}
    \centering
    %Joel Dabrowski
%Tikz Figure
%Name: fig_forecastNetGeneral
%Convolutional ForecastNet Structure

%\def\horizontalsep{1.7cm}
\def\horizontalsep{1.0cm}
\def\verticalsep{0.3cm}

\begin{scriptsize}
	\begin{tikzpicture}[->, >={Latex[length=1.5mm, width=1mm, black!40]}, draw=black!40]
	%\tikzstyle{every pin edge}=[<-,thick]
	%\tikzstyle{output}=[circle, fill=black!20, draw=black!60, minimum size=0.3cm,inner sep=0pt]
	\tikzstyle{output}=[diamond, fill=black!35, draw=black!60, minimum size=0.4cm,inner sep=0pt]
	%\tikzstyle{input}=[circle, fill=black!40, draw=black!60, minimum size=0.3cm,inner sep=0pt]
	\tikzstyle{input}=[circle, fill=black!20, draw=black!60, minimum size=0.3cm,inner sep=0pt]
	\tikzstyle{neuron}=[fill=black!0, draw=black, rotate=90, minimum width = 1.5cm, minimum height=18pt,inner sep=1pt, align=center]
	\tikzstyle{dots}=[inner sep=0pt]
	\pgfsetshortenstart{0.7pt}
	\pgfsetshortenend{0.7pt}
	
	%Inputs
	\node[input, pin={[pin edge={<-}]above:$\mathbf{x}$}] (X1) at (0*\horizontalsep, 0*\verticalsep) {};
	
	\node[dots] (dots1) at (0*\horizontalsep, -1*\verticalsep) {$\dots$};
	
	%Hidden layer 1
	\node[draw=black!15, minimum width=5.1*\horizontalsep, minimum height=1.7cm] (bdr) at 
	(3*\horizontalsep, -1*\verticalsep) {};
	\node[neuron] (H1a) at (1*\horizontalsep, -1*\verticalsep) {Conv. \\ \textcolor{gray}{(f=24, k=2)}};
	\node[neuron] (H1b) at (2*\horizontalsep, -1*\verticalsep) {Pool \\ \textcolor{gray}{(p=2, s=1)}};
	\node[neuron] (H1c) at (3*\horizontalsep, -1*\verticalsep) {Conv. \\ \textcolor{gray}{(f=24, k=2)}};
	\node[neuron] (H1d) at (4*\horizontalsep, -1*\verticalsep) {Pool \\ \textcolor{gray}{(p=2, s=1)}};
	\node[neuron] (H1e) at (5*\horizontalsep, -1*\verticalsep) {Dense \\ \textcolor{gray}{(h=24)}};
	
	%Outputs
	\node[output, pin={[pin edge={->}]below:$\hat{x}_{t+i}$}] (Y1) at (6*\horizontalsep, -2*\verticalsep) {};

	\node[dots] (dots2) at (6*\horizontalsep, -1*\verticalsep) {$\dots$};
	
	%Paths
	\path (dots1) edge (H1a);
	\path (X1) edge (H1a.north);
	\path (H1a) edge (H1b);
	\path (H1b) edge (H1c);
	\path (H1c) edge (H1d);
	\path (H1d) edge (H1e);
	\path (H1e.south) edge (Y1.west);
	\path (H1e) edge (dots2);
	%

%	% Inputs, hidden cells and outputs labels
%	\draw [-, decorate,decoration={brace, amplitude=5pt,raise=0pt},black] 
%	(6.5*\horizontalsep,0.9*\verticalsep) -- (6.5*\horizontalsep,-0.15*\verticalsep) 
%	node [black, midway, xshift=8pt, right, label={[rotate=90]center:Input}] {};
%	
%	\draw [-, decorate,decoration={brace, amplitude=5pt,raise=0pt},black] 
%	(6.5*\horizontalsep,-0.2*\verticalsep) -- (6.5*\horizontalsep,-1.8*\verticalsep) 
%	node [black,midway, xshift=12pt, right, label={[rotate=90,align=center]center:{Hidden \\ cell}}] {};
%%	
%	\draw [-, decorate,decoration={brace, amplitude=5pt,raise=0pt},black] 
%	(6.5*\horizontalsep,-1.85*\verticalsep) -- (6.5*\horizontalsep,-3.1*\verticalsep) 
%	node [black,midway,xshift=8pt, right, label={[rotate=90]center:Output}] {};
	
	\end{tikzpicture}
\end{scriptsize}
    \caption{ForecastNet with a CNN hidden cell structure as used in cFN and cFN2. The number of filters, kernel size, padding, strides, and number of hidden neurons are denoted by f, k, p, s, and h respectively. }%See together with \figurename{~\ref{fig:forecastNet}}.} 
    \label{fig:cFN}
\end{figure}
%

%\section{Additional Dataset Details}
%Four freely available real-world datasets from Rob Hyndman's well-known \textit{Time Series Data Library} (TSDL) \citep{datamarket2018} were selected with the required properties. The first dataset contains monthly England temperature readings spanning 1723 to 1970. The second dataset contains monthly river flow readings of the Chang Jiang river at Hankou, China, spanning 1865 to 1979. The third dataset contains ozone readings spanning 1932 to 1972 over Arosa. The fourth dataset contains monthly Lake Erie level readings from 1921 to 1970. 
%
%An energy dataset containing hourly readings of power consumption in New South Wales, Australia over the period of 2010 and 2011 \citep{aemo2019} is selected. Three aquaculture prawn pond water quality variable datasets \cite{dabrowski2018state} containing sensor readings of dissolved oxygen, pH, and water temperature are used. These datasets were originally sampled at 15 minute intervals, but were resampled to hourly intervals by averaging over the hour. Finally, the Metro Interstate Traffic Volume Dataset from the UCI machine learning repository \cite{uci2019} is selected.

\section{Detailed Derivation of Interleaved Output Chain Rule in Equation (\ref{eq:chainRule2})}
\label{sec:interleavedOutput}
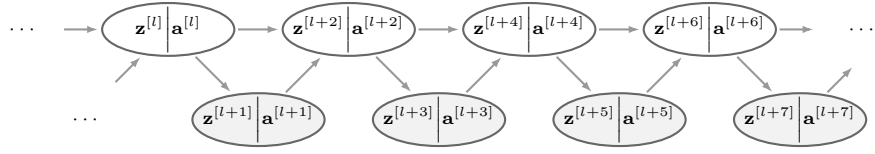
\begin{figure*}[!t]
    \centering
    %Joel Dabrowski
%Tikz Figure
%Name: fig_chainRuleDerivation
%ForecastNet figure for the derivation of the chain rule of calculus

\def\neuronsep{1.2cm}

\begin{scriptsize}
\begin{tikzpicture}[->,>={Latex[length=1.5mm, width=1mm, black!40]},draw=black!40, thick]
	\tikzstyle{every pin edge}=[<-]
	\tikzstyle{neuron}=[ellipse, fill=black!0, draw=black!60, minimum width=50pt, minimum height=20pt,inner sep=0pt]
	\tikzstyle{tapped neuron}=[ellipse, fill=black!5, draw=black!60, minimum width=50pt, minimum height=20pt, thick]
	\pgfsetshortenstart{1pt}
	\pgfsetshortenend{1pt}
	
%	\node[neuron, pin=left:$x_{t-n_I : t}$] (L1) at (0*\neuronsep, 0) {};
%	\node[tapped neuron,pin={[pin edge={->}]below:$x_{t+1}$}] (L2) at (1*\neuronsep, -\neuronsep) {};
%	\path (L1) edge (L2);
	\node[] (Ld1) at (0.4*\neuronsep,0) {$\cdots$};
	\node[] (Ld2) at (1.1*\neuronsep,-\neuronsep) {$\cdots$};
	\node[ellipse, minimum width=55pt, minimum height=30pt] (L1) at (0*\neuronsep,0) {};
	\node[ellipse, minimum width=55pt, minimum height=30pt] (L2) at (1*\neuronsep,-\neuronsep) {};
	
	\node[neuron, label=center:$\mathbf{z}^{[l]} \bigg| \mathbf{a}^{[l]}$] (L3) at (2*\neuronsep, 0) {};
	\node[tapped neuron, label=center:$\mathbf{z}^{[l+1]} \bigg| \mathbf{a}^{[l+1]}$] (L4) at (3*\neuronsep, -\neuronsep) {};
	\path (L3) edge (L4);
	\path (L2) edge (L3);
	\path (L1) edge (L3);
	
	\node[neuron, label=center:$\mathbf{z}^{[l+2]} \bigg| \mathbf{a}^{[l+2]}$] (L5) at (4*\neuronsep, 0) {};
	\node[tapped neuron, label=center:$\mathbf{z}^{[l+3]} \bigg| \mathbf{a}^{[l+3]}$] (L6) at (5*\neuronsep, -\neuronsep) {};
	\path (L5) edge (L6);
	\path (L4) edge (L5);
	\path (L3) edge (L5);
	
	\node[neuron, label=center:$\mathbf{z}^{[l+4]} \bigg| \mathbf{a}^{[l+4]}$] (L7) at (6*\neuronsep, 0) {};
	\node[tapped neuron, label=center:$\mathbf{z}^{[l+5]} \bigg| \mathbf{a}^{[l+5]}$] (L8) at (7*\neuronsep, -\neuronsep) {};
	\path (L7) edge (L8);
	\path (L6) edge (L7);
	\path (L5) edge (L7);
	
	\node[neuron, label=center:$\mathbf{z}^{[l+6]} \bigg| \mathbf{a}^{[l+6]}$] (L9) at (8*\neuronsep, 0) {};
	\node[tapped neuron, label=center:$\mathbf{z}^{[l+7]} \bigg| \mathbf{a}^{[l+7]}$] (L10) at (9*\neuronsep, -\neuronsep) {};
	\path (L9) edge (L10);
	\path (L8) edge (L9);
	\path (L7) edge (L9);
	
	\node[] (Ld2) at (9.7*\neuronsep,0) {$\cdots$};
	\node[ellipse, minimum width=55pt, minimum height=30pt] (L11) at (10*\neuronsep,0) {};
	\path (L9) edge (L11);
	\path (L10) edge (L11);
	
%	\draw[draw=black,thick,dashed,rotate=45] (0.6*\neuronsep,-\neuronsep) ellipse(0.1 cm and 0.8 cm);
%	\draw[draw=black,thick,dashed,rotate=45] (3.6*\neuronsep,-3*\neuronsep) ellipse(0.1 cm and 0.8 cm);
	
\end{tikzpicture}
\end{scriptsize}
    \caption{Figure for the derivation of equation (\ref{eq:chainRule2}). The top row of nodes are hidden neurons. The bottom row of nodes are output neurons.}
    \label{fig:chainRuleDerivation}
\end{figure*}

The outputs in ForecastNet are interleaved between hidden cells. Consider an $L$-layered ForecastNet with a single hidden neuron in the hidden cell and a single linear output neuron as illustrated in \figurename{~\ref{fig:chainRuleDerivation}}. The inputs are not shown as they do not contribute to the backpropagated error across hidden cells. For some layer $l$ in this network, $W^{[l]}$ is the weight matrix, $\bar{b}^{[l]}$ is the bias vector, $\mathbf{a}^{[l]}$ is the output vector, and $\mathbf{z}^{[l]}=W^{[l]T} \mathbf{a}^{[l-1]} + \bar{b}^{[l]}$. Using the chain rule of calculus, the derivative of the loss function $\mathcal{L}$ with respect to the weights $W^{[l]}$ at layer $l$ is given by
\begin{align*}
\frac{\partial \mathcal{L}}{\partial W^{[l]}} 
%&= \frac{\partial \mathcal{L}}{\partial \mathbf{z}^{[l]}} \frac{\partial \mathbf{z}^{[l]}}{\partial W^{[l]}} \nonumber \\
&= 	\frac{\partial \mathcal{L}}{\partial \mathbf{a}^{[l]}} 
\frac{\partial \mathbf{a}^{[l]}}{\partial W^{[l]}}
\end{align*}

Layer $l$ links to layers $l+1$ and $l+2$. Thus, the derivative with respect to $\mathbf{a}^{[l]}$ is expanded as follows
\begin{align*}
\frac{\partial \mathcal{L}}{\partial W^{[l]}} 
&= \left(
\frac{\partial \mathcal{L}}{\partial \mathbf{z}^{[l+1]}} \frac{\partial \mathbf{z}^{[l+1]}}{\partial \mathbf{a}^{[l]}} +
\frac{\partial \mathcal{L}}{\partial \mathbf{z}^{[l+2]}} \frac{\partial \mathbf{z}^{[l+2]}}{\partial \mathbf{a}^{[l]}}
\right)
\frac{\partial \mathbf{a}^{[l]}}{\partial W^{[l]}} \\
&= 	
\frac{\partial \mathcal{L}}{\partial \mathbf{z}^{[l+1]}} 
\frac{\partial \mathbf{z}^{[l+1]}}{\partial \mathbf{a}^{[l]}}
\frac{\partial \mathbf{a}^{[l]}}{\partial W^{[l]}}
+
\frac{\partial \mathcal{L}}{\partial \mathbf{z}^{[l+2]}} 
\frac{\partial \mathbf{z}^{[l+2]}}{\partial \mathbf{a}^{[l]}}
\frac{\partial \mathbf{a}^{[l]}}{\partial W^{[l]}}
\end{align*}
The term $\nicefrac{\partial \mathcal{L}}{\partial \mathbf{z}^{[l+1]}}$ is computed with respect to a target value as layer $l+1$ is an output layer. Layer $l+2$ is a hidden layer. The term  $\nicefrac{\partial \mathcal{L}}{\partial \mathbf{a}^{[l+2]}}$ is thus expanded as follows.
\begin{align*}
\frac{\partial \mathcal{L}}{\partial W^{[l]}} 
& = 
\frac{\partial \mathcal{L}}{\partial \mathbf{z}^{[l+1]}} 
\frac{\partial \mathbf{z}^{[l+1]}}{\partial \mathbf{a}^{[l]}}
\frac{\partial \mathbf{a}^{[l]}}{\partial W^{[l]}} \\
&~~~~~ +
\left(
\frac{\partial \mathcal{L}}{\partial \mathbf{z}^{[l+3]}} \frac{\partial \mathbf{z}^{[l+3]}}{\partial \mathbf{a}^{[l+2]}} +
\frac{\partial \mathcal{L}}{\partial \mathbf{z}^{[l+4]}} \frac{\partial \mathbf{z}^{[l+4]}}{\partial \mathbf{a}^{[l+2]}}
\right) \\
&~~~~~ \times \left( \frac{\partial \mathbf{z}^{[l+2]}}{\partial \mathbf{a}^{[l]}}
\frac{\partial \mathbf{a}^{[l]}}{\partial W^{[l]}} \right) \\
\end{align*}
Which is expanded into the sum
\begin{align*}
\frac{\partial \mathcal{L}}{\partial W^{[l]}} 
&=
\frac{\partial \mathcal{L}}{\partial \mathbf{z}^{[l+1]}} 
\frac{\partial \mathbf{z}^{[l+1]}}{\partial \mathbf{a}^{[l]}}
\frac{\partial \mathbf{a}^{[l]}}{\partial W^{[l]}} \\
&~~~~~ +
\frac{\partial \mathcal{L}}{\partial \mathbf{z}^{[l+3]}} 
\frac{\partial \mathbf{z}^{[l+3]}}{\partial \mathbf{a}^{[l+2]}} 
\frac{\partial \mathbf{z}^{[l+2]}}{\partial \mathbf{a}^{[l]}}
\frac{\partial \mathbf{a}^{[l]}}{\partial W^{[l]}}\\
&~~~~~ +
\frac{\partial \mathcal{L}}{\partial \mathbf{z}^{[l+4]}} 
\frac{\partial \mathbf{z}^{[l+4]}}{\partial \mathbf{a}^{[l+2]}}
\frac{\partial \mathbf{z}^{[l+2]}}{\partial \mathbf{a}^{[l]}}
\frac{\partial \mathbf{a}^{[l]}}{\partial W^{[l]}}
\end{align*}
Similarly, layer $l+3$ is an output layer and layer $l+4$ is a hidden layer. The term  $\nicefrac{\partial \mathcal{L}}{\partial \mathbf{a}^{[l+4]}}$ is expanded as follows
\begin{align*}
\frac{\partial \mathcal{L}}{\partial W^{[l]}} 
&= 	
\frac{\partial \mathcal{L}}{\partial \mathbf{z}^{[l+1]}} 
\frac{\partial \mathbf{z}^{[l+1]}}{\partial \mathbf{a}^{[l]}}
\frac{\partial \mathbf{a}^{[l]}}{\partial W^{[l]}} \\
& ~~~~ +
\frac{\partial \mathcal{L}}{\partial \mathbf{z}^{[l+3]}} 
\frac{\partial \mathbf{z}^{[l+3]}}{\partial \mathbf{a}^{[l+2]}} 
\frac{\partial \mathbf{z}^{[l+2]}}{\partial \mathbf{a}^{[l]}}
\frac{\partial \mathbf{a}^{[l]}}{\partial W^{[l]}} \\
& ~~~~ +
\left(
\frac{\partial \mathcal{L}}{\partial \mathbf{z}^{[l+5]}} \frac{\partial \mathbf{z}^{[l+5]}}{\partial \mathbf{a}^{[l+4]}} +
\frac{\partial \mathcal{L}}{\partial \mathbf{z}^{[l+6]}} \frac{\partial \mathbf{z}^{[l+6]}}{\partial \mathbf{a}^{[l+4]}}
\right) \\
& ~~~~ \times 
\left( \frac{\partial \mathbf{z}^{[l+4]}}{\partial \mathbf{a}^{[l+2]}}
\frac{\partial \mathbf{z}^{[l+2]}}{\partial \mathbf{a}^{[l]}}
\frac{\partial \mathbf{a}^{[l]}}{\partial W^{[l]}} \right) \\
\end{align*}
Which is expanded into the sum
\begin{align*}
\frac{\partial \mathcal{L}}{\partial W^{[l]}} 
&= 	
\frac{\partial \mathcal{L}}{\partial \mathbf{z}^{[l+1]}} 
\frac{\partial \mathbf{z}^{[l+1]}}{\partial \mathbf{a}^{[l]}}
\frac{\partial \mathbf{a}^{[l]}}{\partial W^{[l]}} \\
& ~~~~ +
\frac{\partial \mathcal{L}}{\partial \mathbf{z}^{[l+3]}} 
\frac{\partial \mathbf{z}^{[l+3]}}{\partial \mathbf{a}^{[l+2]}} 
\frac{\partial \mathbf{z}^{[l+2]}}{\partial \mathbf{a}^{[l]}}
\frac{\partial \mathbf{a}^{[l]}}{\partial W^{[l]}} \\
& ~~~~ +
\frac{\partial \mathcal{L}}{\partial \mathbf{z}^{[l+5]}} 
\frac{\partial \mathbf{z}^{[l+5]}}{\partial \mathbf{a}^{[l+4]}} 
\frac{\partial \mathbf{z}^{[l+4]}}{\partial \mathbf{a}^{[l+2]}}
\frac{\partial \mathbf{z}^{[l+2]}}{\partial \mathbf{a}^{[l]}}
\frac{\partial \mathbf{a}^{[l]}}{\partial W^{[l]}} \\
& ~~~~ +
\frac{\partial \mathcal{L}}{\partial \mathbf{z}^{[l+6]}} 
\frac{\partial \mathbf{z}^{[l+6]}}{\partial \mathbf{a}^{[l+4]}}
\frac{\partial \mathbf{z}^{[l+4]}}{\partial \mathbf{a}^{[l+2]}}
\frac{\partial \mathbf{z}^{[l+2]}}{\partial \mathbf{a}^{[l]}}
\frac{\partial \mathbf{a}^{[l]}}{\partial W^{[l]}}
\end{align*}
This expansion process is continued until the final output layer $L$ is reached. The final result is 
\begin{align*}
\label{eq:chainRule2}
\frac{\partial \mathcal{L}}{\partial W^{[l]}} 
&= 	
\sum_{k=0}^{\frac{L-1-l}{2}}
\frac{\partial \mathcal{L}}{\partial \mathbf{z}^{[l+2k+1]}} 
\frac{\partial \mathbf{z}^{[l+2k+1]}}{\partial \mathbf{a}^{[l+2k]}}
\Psi_k
\frac{\partial \mathbf{a}^{[l]}}{\partial W^{[l]}}
\end{align*}
where
\begin{align*}
%~~ \text{where} ~~
\Psi_k = 
\begin{cases}
1 & k=0 \\
\displaystyle \prod_{j=1}^{k} \frac{\partial \mathbf{z}^{[l+2j]}}{\partial \mathbf{a}^{[l+2(j-1)]}}  & k>0
\end{cases}
\end{align*}

    \else
    \fi

\end{document}